\documentclass{article}

\usepackage[preprint]{ProbNum26}

\usepackage[round]{natbib}

\usepackage{amsmath,amssymb,amsthm}
\usepackage{hyperref}
\usepackage{booktabs}
\usepackage{bm}
\usepackage{kantlipsum}
\usepackage[capitalise,nameinlink]{cleveref}
\usepackage{tikz}
\usetikzlibrary{positioning, arrows.meta, shapes.geometric, fit, backgrounds, calc}

\newtheorem{theorem}{Theorem}

\newtheorem{proposition}{Proposition}
\newtheorem{corollary}{Corollary}
\newtheorem{definition}{Definition}
\newtheorem{assumption}{Assumption}
\newtheorem{remark}{Remark}

\newcommand{\R}{\mathbb{R}}
\newcommand{\E}{\mathbb{E}}
\newcommand{\Var}{\mathrm{Var}}
\newcommand{\KL}{\mathrm{KL}}

\newcommand{\Cov}{\mathrm{Cov}}
\DeclareMathOperator{\diag}{diag}

\probnumtitle{Three Costs of Amortizing Gaussian Process Inference with Neural Processes}
\probnumauthors{%
\name{Robin Young}%
\affiliation{University of Cambridge, Cambridge, UK}
}

\probnumabstract{Neural processes amortize Gaussian process inference, replacing the exact $O(n^3)$ posterior with a learned $O(n)$ map from context sets to predictive distributions. For a class of latent neural processes, we bound the Kullback--Leibler (KL) divergence between the GP and LNP predictives, decomposing it into three interpretable sources, namely label contamination as the neural process uses label values to estimate a quantity that is label-independent in the exact GP, an information bottleneck because the finite-dimensional representation cannot resolve the full context geometry, and amortization error from a single encoder network shared across all contexts. The bottleneck truncation term decays in the representation dimension $d$ as $O(e^{-cd^{2/d_x}})$ for squared-exponential kernels on $\mathbb{R}^{d_x}$ where $c > 0$ is a kernel-dependent constant and as $O(d^{-2\nu/d_x})$ for Mat\'ern-$\nu$ kernels, directly linking architecture sizing to kernel smoothness and input dimension. The label contamination term is $O(1)$ in general, with only the observation-noise component decaying as $O(1/n)$, identifying a persistent cost of routing uncertainty estimation through a label-dependent representation. These results characterize the costs of amortization within the analyzed class and yield architectural recommendations to predict variance from context locations alone in the GP-amortization regime, and replace mean aggregation with second-order pooling to close the dominant amortization gap.}

\begin{document}

\section{Introduction}
 
Gaussian processes \citep{rasmussen2005gaussian} define exact posterior distributions over functions given observed data. The posterior mean and variance are available in closed form, but computing them costs $O(n^3)$ in the number of context points $n$\footnote{Specifically, $O(n^3)$ for the one-time Cholesky factorization of $K + \sigma_\epsilon^2 I$, subsequent predictions cost $O(n^2)$ per query.}, limiting scalability.
 
Neural processes \citep{garnelo2018cnp, garnelo2018neuralprocesses} address this by learning an amortized map from context sets to predictive distributions. Subsequent work has extended the framework through attention mechanisms \citep{kim2018attentive}, convolutional structure \citep{Gordon2020Convolutional}, and autoregressive factorizations \citep{bruinsma2023autoregressive}. A latent neural process (LNP) maps each context pair through an encoder, aggregates the resulting features into a finite-dimensional representation, passes this representation through a recognition network producing a latent distribution, and decodes samples from the latent into a predictive distribution, all in $O(n)$ time for fixed architecture. When trained on samples from a GP prior, the LNP is implicitly approximating the GP posterior.

Sparse variational methods \citep{titsias2009sparse} offer one path to scaling GP inference, reducing the cubic cost to $O(M^2 n)$ by summarizing the posterior through $M$ inducing points. However, the resulting predictive still requires $O(M^2)$ computation at test time, and the inducing locations must be optimized jointly with kernel hyperparameters for each new dataset. In settings that demand predictions across many related tasks or under real-time constraints such as sequential experimental design \citep{shahriari2016taking}, robotics \citep{deisenroth2015gaussian}, or simulation-based inference \citep{cranmer2020frontier} this per-task optimization becomes a bottleneck of its own. Neural processes eliminate per-task optimization entirely with a single forward pass mapping any context set to a predictive distribution in $O(n)$ time, with no matrix inversion and no task-specific parameters. This amortization is the source of their practical appeal, but it introduces approximation costs that have no counterpart in the sparse GP framework.

It is known that this approximation introduces error as the variational inference underlying the LNP tends to underestimate uncertainty, and the finite-dimensional bottleneck cannot capture arbitrary context geometries. However, no quantitative characterization of these errors exists. \citet{foong2020meta} identify qualitative limitations of mean aggregation, showing it leads to underfitting, but do not provide rates. People choose representation dimensions, encoder architectures, and variance parameterizations without formal guidance on the resulting approximation quality. In the sparse GP literature, \citet{burt2020convergence} bound the KL divergence between sparse variational and exact GP posteriors, showing that $M = O((\log N)^{d_x})$ inducing points suffice for squared-exponential kernels. Our bottleneck analysis plays an analogous role for neural processes, with the representation dimension $d$ replacing the number of inducing points $M$.
 
We provide such a characterization. Our main contributions are: (i) a decomposition of the predictive KL divergence $\KL(p_{\mathrm{GP}} \| p_{\mathrm{LNP}})$ into three terms with distinct architectural and statistical origins; (ii) an upper bound on the truncation component of the bottleneck term that decays exponentially in $d$ for SE kernels and polynomially for Mat\'ern kernels, connecting representation dimension to kernel smoothness; (iii) characterization of the label contamination term as $O(1) + O(1/n)$, identifying a structural mismatch in how neural processes estimate predictive variance; and (iv) two architectural recommendations to predict variance from locations alone in the GP-amortization regime, and use second-order rather than mean aggregation.

\section{Setup}
 
\subsection{Gaussian Process}
 
Let $f \sim \mathcal{GP}(0, k)$ be a zero-mean GP with kernel $k: \mathcal{X} \times \mathcal{X} \to \R$ on a compact domain $\mathcal{X} \subset \R^{d_x}$. Given context $C = \{(x_i, y_i)\}_{i=1}^n$ with $y_i = f(x_i) + \epsilon_i$, $\epsilon_i \sim \mathcal{N}(0, \sigma_\epsilon^2)$, the GP predictive at a target $x_*$ is:
\begin{equation}
    p_{\mathrm{GP}}(y_* \mid x_*, C) = \mathcal{N}(\mu_{\mathrm{GP}}, \sigma^2_{\mathrm{GP}})
\end{equation}
where
\begin{align}
    \mu_{\mathrm{GP}} &= \bm{k}_*^\top (K + \sigma_\epsilon^2 I)^{-1} \bm{y}, \quad \bm{y} = (y_1, \dots, y_n)^\top \label{eq:gp_mean} \\
    \sigma^2_{\mathrm{GP}} &= k(x_*, x_*) - \bm{k}_*^\top (K + \sigma_\epsilon^2 I)^{-1} \bm{k}_*, \label{eq:gp_var}
\end{align}
with $[\bm{k}_*]_i = k(x_*, x_i)$ and $[K]_{ij} = k(x_i, x_j)$.
 
A structural property is that the GP predictive variance \eqref{eq:gp_var} depends on the context locations $X = \{x_i\}$ but not on the labels $\bm{y}$.
 
\subsection{Latent Neural Process}
 
\begin{definition}[Latent Neural Process]\label{def:lnp}
A latent neural process with mean aggregation consists of:
\begin{enumerate}
    \item An encoder $h: \mathcal{X} \times \mathcal{Y} \to \R^d$ mapping context pairs to representations.
    \item Mean aggregation: $r_C = \frac{1}{n}\sum_{i=1}^n h(x_i, y_i)$.
    \item A latent encoder: $q(z \mid C) = \mathcal{N}(\mu_z(r_C), \Sigma_z(r_C))$ with $z \in \R^{d_z}$.
    \item A decoder: $p(y_* \mid x_*, z) = \mathcal{N}(w(x_*)^\top z + b(x_*), \sigma_d^2)$.
\end{enumerate}
The maps $\mu_z: \R^d \to \R^{d_z}$ and $\Sigma_z: \R^d \to \R^{d_z \times d_z}_{\succ 0}$ are typically MLPs with the covariance output parameterized via a Cholesky factor to enforce positive definiteness. The functions $w: \mathcal{X} \to \R^{d_z}$ and $b: \mathcal{X} \to \R$ are learned (typically MLP-parameterized) functions of the target location, not network weights and biases. The linear-in-$z$ decoder is a tractability choice that yields a closed-form Gaussian posterior $p(z \mid C)$ amenable to the amortization-gap analysis of \Cref{sec:amort}; standard NPs use nonlinear MLP decoders. See \Cref{app:np_primer} for a fuller introduction to NPs in case they are unfamiliar.
\end{definition}

The marginal predictive is Gaussian:
\begin{align}
    p_{\mathrm{LNP}}(y_* \mid x_*, C) &= \mathcal{N}(\mu_{\mathrm{LNP}}, \sigma^2_{\mathrm{LNP}}), \\
    \mu_{\mathrm{LNP}} &= w(x_*)^\top \mu_z(r_C) + b(x_*), \label{eq:lnp_mean} \\
    \sigma^2_{\mathrm{LNP}} &= w(x_*)^\top \Sigma_z(r_C)\, w(x_*) + \sigma_d^2. \label{eq:lnp_var}
\end{align}
 
Unlike \eqref{eq:gp_var}, the LNP variance \eqref{eq:lnp_var} depends on $r_C$, which encodes both context locations and labels.
 
\subsection{Predictive KL Divergence}
 
Both predictives are Gaussian, so the KL divergence has a closed form:
\begin{equation}
\small
    \KL(p_{\mathrm{GP}} \| p_{\mathrm{LNP}}) = \frac{1}{2}\left[\frac{\sigma^2_{\mathrm{GP}}}{\sigma^2_{\mathrm{LNP}}} - 1 + \frac{(\mu_{\mathrm{GP}} - \mu_{\mathrm{LNP}})^2}{\sigma^2_{\mathrm{LNP}}} + \log \frac{\sigma^2_{\mathrm{LNP}}}{\sigma^2_{\mathrm{GP}}}\right] \label{eq:kl}
\end{equation}
 
We bound the expected gap $\E_C[\KL(p_{\mathrm{GP}} \| p_{\mathrm{LNP}})]$ where the expectation is over contexts drawn from the GP prior.

\begin{assumption}
\label{ass:bounded_var}
There exist constants $0 < \sigma_\ell^2 \leq \sigma_u^2 < \infty$ such that $\sigma_\ell^2 \leq \sigma^2_{\mathrm{LNP}}(x_*; C) \leq \sigma_u^2$ and $\sigma_\ell^2 \leq \sigma^2_{\mathrm{GP}}(x_*; X) \leq \sigma_u^2$ for all contexts $C$, locations $X$, and targets $x_* \in \mathcal{X}$.
\end{assumption}

\begin{remark}
The GP side holds automatically with $\sigma_\ell^2 = \sigma_\epsilon^2 \kappa_k / (\kappa_k + \sigma_\epsilon^2)$ and $\sigma_u^2 = \kappa_k$ (where $\kappa_k = \sup_{x \in \mathcal{X}} k(x,x)$ bounds the prior signal variance), since adding context points can only reduce posterior variance from the prior. The LNP side is enforced architecturally by requiring $\sigma_d^2 \geq \sigma_\ell^2$ in the decoder and bounding the weights $\|w(x_*)\|$ and the eigenvalues of $\Sigma_z$. This assumption is used in the variance ratio expansion in \Cref{thm:main}.
\end{remark}

\section{Decomposition of the Variance Error}\label{sec:variance_decomp}
 
The variance gap is the primary source of the variational inference (VI) bias. We decompose it into three sources. Define the variance error as $\Delta\sigma^2 = \sigma^2_{\mathrm{LNP}} - \sigma^2_{\mathrm{GP}}$.
 
\subsection{Label Contamination}\label{sec:label_cont}
 
The GP variance $\sigma^2_{\mathrm{GP}}(x_*; X)$ is a function of context locations only. The LNP variance $\sigma^2_{\mathrm{LNP}}(x_*; r_C)$ depends on labels through $r_C$. This mismatch introduces noise in the variance estimate.
 
\begin{assumption}\label{ass:encoder}
The encoder is affine in $y$: $h(x, y) = \phi(x) + \psi(x) y$, with positive constants $B_\phi, B_\psi$ such that $\|\phi(x)\| \leq B_\phi$ and $\|\psi(x)\| \leq B_\psi$ for all $x \in \mathcal{X}$.
\end{assumption}
 
Under this assumption, the representation decomposes as:
\begin{equation}\label{eq:rep_decomp}
    r_C = \underbrace{\frac{1}{n}\sum_{i=1}^n \phi(x_i)}_{\bar{\phi}_X} + \underbrace{\frac{1}{n}\sum_{i=1}^n \psi(x_i) y_i}_{\delta_y}.
\end{equation}
 
The first term depends only on locations; the second introduces label dependence.
 
\begin{assumption}\label{ass:lipschitz}
The maps $\mu_z: \R^d \to \R^{d_z}$ and $\Sigma_z: \R^d \to \R^{d_z \times d_z}$ are $L_\mu$-Lipschitz and $L_\Sigma$-Lipschitz respectively. The decoder satisfies $\|w(x_*)\| \leq B_w$.
\end{assumption}
 
\begin{theorem}[Label Contamination Bound]\label{thm:label}
Under Assumptions~\ref{ass:encoder} and \ref{ass:lipschitz}, for contexts drawn from a GP prior with kernel $k$ and noise variance $\sigma_\epsilon^2$:
\begin{equation}
    \E_{C}\!\left[\Var_{\bm{y}|X}\!\left[\sigma^2_{\mathrm{LNP}}(x_*; C)\right]\right] \leq L_\Sigma^2 B_w^4 B_\psi^2\!\left(\frac{\sigma_\epsilon^2}{n} + \kappa_k\right)
\end{equation}
where $\kappa_k = \sup_{x \in \mathcal{X}} k(x,x)$ bounds the signal variance. The first term arises from observation noise and decays as $O(1/n)$; the second arises from correlations under the GP prior and is $O(1)$.
\end{theorem}

\begin{proof}[Proof sketch]
Decompose the label-dependent component as $\delta_y = \delta_f + \delta_\epsilon$ where $\delta_f = \frac{1}{n}\sum_i \psi(x_i) f(x_i)$ and $\delta_\epsilon = \frac{1}{n}\sum_i \psi(x_i) \epsilon_i$. The noise term is an i.i.d.\ sum, so $\E[\|\delta_\epsilon\|^2 \mid X] \leq B_\psi^2 \sigma_\epsilon^2/n$. The signal term has covariance $\frac{1}{n^2}\sum_{i,j} \psi(x_i)\psi(x_j)^\top k(x_i, x_j)$. Since $k$ is positive definite with bounded diagonal $\kappa_k$, the row-sum bound gives $\E[\|\delta_f\|^2 \mid X] \leq B_\psi^2 \kappa_k$, which is $O(1)$ because the GP-correlated $f(x_i)$ do not concentrate. Lipschitz continuity of $\Sigma_z$ then transfers these bounds to $\sigma^2_{\mathrm{LNP}}$. Full proof in \Cref{app:proof_label}.
\end{proof}

\begin{remark}
Only the noise component decays as $O(1/n)$. The signal component is $O(1)$ because the $f(x_i)$ are correlated under the GP prior. $\delta_f = \frac{1}{n}\sum_i \psi(x_i) f(x_i)$ converges to $\E_\mu[\psi(x)f(x)]$, which is a nondegenerate random variable across GP draws. This $O(1)$ term is irreducible within architectures that route uncertainty estimation through a label-dependent bottleneck, and persists even as $n \to \infty$.
\end{remark}

\subsection{Information Bottleneck}\label{sec:bottleneck}
 
Even removing label dependence (conditioning on $\bm{y}$ or taking $\bar{\phi}_X$ as the representation), the $d$-dimensional summary $\bar{\phi}_X$ cannot capture the full context geometry that determines the GP variance.\footnote{The truncation component of this bottleneck error follows from standard Mercer approximation theory. We include it as the baseline against which the additional NP-specific costs (label contamination, amortization gap) are measured. These additional costs have no analogue in sparse variational GP theory.}

\begin{definition}[Mercer Expansion]
Let $\{(\lambda_j, e_j)\}_{j=1}^\infty$ be the eigendecomposition of the kernel integral operator $T_k$ with respect to a base measure $\mu$ on $\mathcal{X}$, taken throughout to coincide with the distribution from which context locations are drawn:
\begin{equation}
    k(x, x') = \sum_{j=1}^\infty \lambda_j\, e_j(x)\, e_j(x').
\end{equation}
\end{definition}
 
\begin{assumption}\label{ass:mercer_encoder}
The location encoder uses the first $d$ kernel eigenfunctions $\phi(x) = (\sqrt{\lambda_1}\, e_1(x), \dots, \sqrt{\lambda_d}\, e_d(x))^\top$.
\end{assumption}
 
This is the optimal $d$-dimensional encoder in the sense of minimizing integrated squared error in kernel approximation. The mean representation becomes $\bar{\phi}_X = \frac{1}{n}\sum_i \phi(x_i)$, which captures the projection of the empirical context measure onto the first $d$ kernel eigenfunctions.
 
\begin{assumption}\label{ass:tail}
The representation dimension $d$ is large enough that the tail operator norm satisfies
\begin{equation}
    \eta(X) \;=\; \bigl\|A^{-1/2}\, K_T\, A^{-1/2}\bigr\|_{\mathrm{op}} \;<\; 1,
\end{equation}
where $A = K_H + \sigma_\epsilon^2 I$ with head kernel matrix $[K_H]_{il} = \sum_{j \leq d}\lambda_j\,e_j(x_i)\,e_j(x_l)$, and tail kernel matrix $[K_T]_{il} = \sum_{j>d}\lambda_j\,e_j(x_i)\,e_j(x_l)$.
\end{assumption}
 
\begin{remark}
Since $A \succeq \sigma_\epsilon^2 I$, the condition $\eta < 1$ is implied by $\|K_T\|_{\mathrm{op}} < \sigma_\epsilon^2$. For i.i.d.\ context points, $\E[\|K_T\|_{\mathrm{op}}] \leq n\lambda_T$ where $\lambda_T = \sum_{j>d}\lambda_j$, so the assumption holds (with high probability) once $n\lambda_T < \sigma_\epsilon^2$. For SE kernels this requires $d = \Omega((\log n)^{d_x/2})$\footnote{Throughout, $\Omega(\cdot)$ and $\Theta(\cdot)$ denote standard Bachmann--Landau lower-bound and tight-bound notation respectively.} for Mat\'ern-$\nu$ kernels, $d = \Omega(n^{d_x/(2\nu)})$. This is the natural regime in which the $d$-truncated GP is a meaningful approximation: the tail kernel contributes less than the observation noise.
\end{remark}

\begin{theorem}[Information Bottleneck Bound]\label{thm:bottleneck}
Under Assumptions~\ref{ass:mercer_encoder} and~\ref{ass:tail}, for a target $x_*$ and context locations $X = \{x_i\}_{i=1}^n$ drawn i.i.d.\ from $\mu$:
\begin{equation}
    \E_X\!\left[\left(\sigma^2_{\mathrm{GP}}(x_*; X) - g_d^*(\bar{\phi}_X)\right)^2\right] \leq R_d(k, n) + S_d(k, x_*),
\end{equation}
where $g_d^*$ is the optimal estimator of $\sigma^2_{\mathrm{GP}}$ from $\bar{\phi}_X$, and:
\begin{align}
    R_d(k, n) &= O\!\left(\frac{d^2}{n}\right) + R_d^{\mathrm{info}}, \label{eq:estimation}\\
    S_d(k, x_*) &\leq C_S \sum_{j > d} \lambda_j\, e_j(x_*)^2, \label{eq:truncation}
\end{align}
with $C_S = C\,\kappa_k / \sigma_\epsilon^4$ for a constant $C$ depending only on $\kappa_k/\sigma_\epsilon^2$. Here $R_d^{\mathrm{info}}$ is the irreducible estimation loss from $\bar{\phi}_X$ being an insufficient statistic for the head kernel matrix, and the $O(d^2/n)$ term is the variance of $\bar{\phi}_X$ about its population mean. For second-order aggregation (using $\frac{1}{n}\sum_i \phi(x_i)\phi(x_i)^\top$ in place of $\bar{\phi}_X$), $R_d^{\mathrm{info}} = 0$ and the estimation error is $O(d^2/n)$.
\end{theorem}
 
\begin{proof}[Proof sketch]
Split the Mercer series into head ($j \leq d$) and tail ($j > d$) components, with $A = K_H + \sigma_\epsilon^2 I$ and $M = A + K_T$. The truncated GP variance $\sigma_d^2 = k_H - (\bm{k}_*^H)^\top A^{-1} \bm{k}_*^H$ depends only on head eigenfunction evaluations. Since the optimal estimator $g_d^*(\bar\phi_X)$ minimizes MSE over measurable functions of $\bar\phi_X$, and $\sigma_d^2$ is measurable with respect to a finer $\sigma$-algebra, $S_d$ is bounded by $\E_X[(\sigma^2_{\mathrm{GP}} - \sigma_d^2)^2]$. A Neumann expansion of $M^{-1}$ around $A^{-1}$, valid under \Cref{ass:tail}, gives $|\sigma^2_{\mathrm{GP}} - \sigma_d^2| = O(\sqrt{k_T})$, hence the squared error is $O(k_T)$ with constant $C_S = C\kappa_k/\sigma_\epsilon^4$. The estimation term $R_d$ contains an $O(d^2/n)$ CLT contribution from $\bar\phi_X$ converging to its population mean, plus the irreducible $R_d^{\mathrm{info}}$ from mean aggregation being insufficient for the head kernel matrix; under second-order aggregation $R_d^{\mathrm{info}} = 0$. Full proof in \Cref{app:proof_bottleneck}.
\end{proof}

\begin{remark}[Scope of the bound]
\Cref{thm:bottleneck} bounds the variance error of the Bayes-optimal estimator $g_d^*(\bar\phi_X)$ given the mean representation; this characterizes the information-theoretic limit of any estimator built on $\bar\phi_X$. The cost of replacing $g_d^*$ with the learned amortized estimator is captured separately in \Cref{sec:amort}. The truncation term $S_d$ is fully controlled and inherits the kernel's eigenvalue decay (\Cref{cor:rates}). The estimation term $R_d$ contains a controlled $O(d^2/n)$ component plus an irreducible information loss $R_d^{\mathrm{info}}$ that we do not bound under mean aggregation; under second-order aggregation $R_d^{\mathrm{info}} = 0$.
\end{remark}

The truncation error $S_d$ inherits the eigenvalue decay of the kernel:

\begin{corollary}[Kernel-Dependent Bottleneck Rates]\label{cor:rates}
Under Assumptions~\ref{ass:mercer_encoder} and~\ref{ass:tail}:
\begin{enumerate}
    \item \textbf{Squared-exponential kernel} on $[0,1]^{d_x}$: Eigenvalues satisfy $\lambda_j \leq C_{\mathrm{SE}}\, e^{-c\, j^{2/d_x}}$ for constants $C_{\mathrm{SE}}, c > 0$ depending on the lengthscale $\ell$. Thus:
    \begin{equation}
        S_d \leq C'_{\mathrm{SE}}\, e^{-c'\, d^{2/d_x}}.
    \end{equation}
    \item \textbf{Mat\'ern-$\nu$ kernel} on $[0,1]^{d_x}$: Eigenvalues satisfy $\lambda_j \leq C_\nu\, j^{-(2\nu + d_x)/d_x}$. Thus:
    \begin{equation}
        S_d \leq C'_\nu\, d^{-2\nu/d_x}.
    \end{equation}
\end{enumerate}
The constants $C'_{\mathrm{SE}}$ and $C'_\nu$ absorb the factor $C_S = C\,\kappa_k/\sigma_\epsilon^4$ from \Cref{thm:bottleneck}.
\end{corollary}
 
\begin{proof}
For the SE kernel, the tail sum $\sum_{j>d} \lambda_j e_j(x_*)^2 \leq \|e\|_\infty^2 \sum_{j>d} \lambda_j$. The eigenvalues of the SE kernel on compact domains decay as $\lambda_j \leq C\, e^{-c\, j^{2/d_x}}$ \citep{SteinwartScovel2012}. Thus $\sum_{j>d}\lambda_j \leq C\, e^{-c\, d^{2/d_x}}$. Substituting into the linear truncation bound gives the result.
 
For the Mat\'ern kernel, the eigenvalue asymptotics $\lambda_j = \Theta(j^{-(2\nu+d_x)/d_x})$ follow from the spectral theory of Sobolev-equivalent reproducing kernels \citep{wendland2004scattered}. Then $\sum_{j>d}\lambda_j \leq C \int_d^\infty t^{-(2\nu+d_x)/d_x}\, dt = C'\, d^{-2\nu/d_x}$. Substituting into the linear truncation bound gives $S_d = O(d^{-2\nu/d_x})$.
\end{proof}
 
\begin{remark}
The scaling $d = O((\log(1/\epsilon))^{d_x/2})$ for SE kernels parallels the inducing point requirement $M = O((\log N)^{d_x})$ established by \citet{burt2020convergence} for sparse variational GPs. Both rates are governed by the same eigenvalue decay, reflecting a shared spectral bottleneck. The difference is that the sparse GP approximation introduces only a truncation error (analogous to our $S_d$ term), whereas the neural process additionally incurs label contamination and amortization costs that have no counterpart in the sparse GP framework.

For $d_x = 1$ the SE bottleneck is $O(e^{-cd^2})$, and Mat\'ern-$\nu$ is $O(d^{-2\nu})$. A Mat\'ern-$3/2$ kernel requires $d = O(\epsilon^{-1/3})$ for bottleneck error $\epsilon$; the SE kernel requires only $d = O(\sqrt{\log(1/\epsilon)})$. Rougher kernels demand wider representations. The tail-dominance condition (\Cref{ass:tail}) is mild. For SE kernels it holds once $d = \Omega((\log n)^{d_x/2})$, and for Mat\'ern-$\nu$ kernels once $d = \Omega(n^{d_x/(2\nu)})$.
\end{remark}

\subsection{Amortization Gap}\label{sec:amort}

The LNP uses a single encoder network $(\mu_z(\cdot), \Sigma_z(\cdot))$ for all contexts. For any specific context $C$, the optimal variational parameters may differ from what the amortized encoder produces. The concept of an amortization gap was introduced in the VAE literature by \citet{cremer2018inference}.

\begin{definition}[Amortization Gap]
For a fixed context $C$, let $q^*(z \mid C) = \arg\min_{q \in \mathcal{Q}} \KL(q(z) \| p(z \mid C))$ be the optimal per-task variational distribution. The amortization gap is:
\begin{equation}
\small
    A(C) = \KL(q_{\mathrm{amort}}(z \mid C) \| p(z \mid C)) - \KL(q^*(z \mid C) \| p(z \mid C))
\end{equation}
\end{definition}

Under the linear decoder (item~4 of \Cref{def:lnp}) with prior $p(z) = \mathcal{N}(0, I)$, the posterior is Gaussian:
\begin{equation}\label{eq:posterior}
    p(z \mid C) = \mathcal{N}(\mu_p, \Sigma_p), \quad 
    \Sigma_p^{-1} = I + \sigma_d^{-2}\textstyle\sum_{i=1}^n w(x_i) w(x_i)^\top
\end{equation}
Since the variational family $\mathcal{Q}$ contains all Gaussians, $q^*(z \mid C) = p(z \mid C)$, so $A(C) = \KL(q_{\mathrm{amort}}(z \mid C) \| p(z \mid C))$. Crucially, the posterior covariance $\Sigma_p$ depends on the \emph{second-order} empirical features $\sum_i w(x_i) w(x_i)^\top$, but mean aggregation provides only the \emph{first-order} summary $\bar{\phi}_X = \frac{1}{n}\sum_i \phi(x_i)$. This mismatch is the source of an irreducible amortization gap.

We make this precise in the scalar case and then state the general result.

\begin{proposition}[Amortization Gap for Mean Aggregation, $d_z = d = 1$]\label{prop:amort_scalar}
Let $d_z = d = 1$ with the Mercer encoder $\phi(x) = \sqrt{\lambda_1}\, e_1(x)$ and linear decoder $w(x) = \phi(x)$, where $e_1$ is the first eigenfunction of the kernel integral operator with respect to a measure $\mu$ on $\mathcal{X}$. Let context locations $x_1, \dots, x_n \text{ i.i.d. } \mu$. Define $\alpha = n\lambda_1/\sigma_d^2$. Then the variance component of the amortization gap satisfies:
\begin{align}
\label{eq:amort_exact}
\mathbb{E}_X\!\left[\mathrm{KL}\!\left(\mathcal{N}(0,\,\Sigma_z^*(\bar\phi_X))\;\big\|\;\mathcal{N}(0,\,\Sigma_p(X))\right)\right] \\= \frac{\alpha^2}{8n(1+\alpha)^2}\left(1 + O(1/n)\right).
\end{align}
where $\Sigma_z^*(\bar\phi_X)$ is the optimal posterior covariance estimator given only the mean representation. Asymptotically, $\E_X[A_\sigma(X)] \to \frac{1}{8n}$ as $n\lambda_1/\sigma_d^2 \to \infty$.
\end{proposition}

\begin{proof}[Proof sketch]
The posterior precision is $\Sigma_p^{-1} = 1 + \alpha \hat v$ where $\hat v = \frac{1}{n}\sum_i e_1(x_i)^2$. The mean representation $\bar e = \frac{1}{n}\sum_i e_1(x_i)$ is uncorrelated with $\hat v$ because $\Cov(e_1(x)^2, e_1(x)) = \E_\mu[e_1^3] - \E_\mu[e_1^2]\E_\mu[e_1] = 0$ (both terms vanish under standard symmetric eigenfunction bases). So the Bayes-optimal estimator of $\Sigma_p$ given $\bar e$ uses $\E[\hat v] = 1$, giving $\Sigma_z^* = (1+\alpha)^{-1}$. A second-order expansion of the resulting KL in $\hat v - 1 = O_p(n^{-1/2})$ yields the rate. Full proof in \Cref{app:proof_amort_scalar}.
\end{proof}

\begin{remark}
The gap is $\Theta(1/n)$, the same order as the noise-induced label contamination (\Cref{thm:label}). Unlike the label contamination signal term, which is $O(1)$ and irreducible, the amortization gap is curable: under second-order aggregation, the representation $\frac{1}{n}\sum_i \phi(x_i)\phi(x_i)^\top$ is a sufficient statistic for $\Sigma_p$, and the gap improves to $O(1/n^2)$ (verified empirically in Table~\ref{tab:mean_vs_so}, with fitted slopes $n^{-0.81}$ vs $n^{-1.75}$). This provides a architectural recommendation to replace mean aggregation with second-order pooling to eliminate the dominant amortization cost.
\end{remark}

The result extends to general $d$:

\begin{proposition}[Amortization Gap, General $d$]\label{prop:amort_general}
Under Assumptions~\ref{ass:mercer_encoder} and~\ref{ass:bounded_var}, with $d_z = d$, decoder $w(x) = \phi(x)$, and mean aggregation:
\begin{equation}
    \E_X[A_\sigma(X)] = \frac{1}{n}\sum_{j=1}^d c_j(\lambda_j, \sigma_d^2, n) + O(d^2/n^2),
\end{equation}
where each $c_j = \Theta(1)$ for eigenvalues bounded away from zero, arising from the conditional variance $\Var[\hat{v}_j \mid \bar{e}]$ of the $j$-th diagonal second moment $\hat{v}_j = \frac{1}{n}\sum_i e_j(x_i)^2$ given the mean representation. For the Mercer encoder with decaying eigenvalues:
\begin{equation}
    \E_X[A_\sigma(X)] = \Theta(d_{\mathrm{eff}}/n),
\end{equation}
where $d_{\mathrm{eff}} = \bigl(\sum_j \lambda_j^2\bigr) / \bigl(\sum_j \lambda_j\bigr)^2$ is the spectrum-weighted effective dimension.
\end{proposition}

\begin{proof}[Proof sketch]
The posterior precision is
\begin{align}
    \Sigma_p^{-1} &= I + \frac{n}{\sigma_d^2}\,\Lambda\,\hat{V}\,\Lambda, \nonumber\\
    \hat{V}_{jk} &= \frac{1}{n}\sum_i e_j(x_i)\,e_k(x_i), \nonumber\\
    \Lambda &= \diag(\sqrt{\lambda_1}, \dots, \sqrt{\lambda_d}).
\end{align}
Mean aggregation provides $\bar{e}_j = \frac{1}{n}\sum_i e_j(x_i)$ for $j = 1, \dots, d$, which does not determine $\hat{V}$.

\paragraph{General claim.} The conditional variance $\Var[\hat{v}_j \mid \bar{e}]$ is bounded below by $\Var[\hat{v}_j](1 - \rho_j^2)$ where $\rho_j$ measures how well $\hat{v}_j$ is predicted by linear functionals of $\bar{e}$. Under broad conditions on the eigenfunction basis (orthogonality of degree-2 monomials in the eigenfunctions to the linear span $\{e_1, \dots, e_d\}$ outside a sparse set of indices), $\rho_j$ is bounded away from $1$, so each $j$ contributes $\Theta(1/n)$ to the gap. Weighting by the sensitivity $\partial \Sigma_p / \partial \hat{v}_j \propto \lambda_j^2$ gives the effective-dimension scaling $\Theta(d_{\mathrm{eff}}/n)$.

\paragraph{Worked example cosine basis on $[0,1]$.} For $e_j(x) = \sqrt{2}\cos(j\pi x)$ under the uniform measure,
\begin{equation}
    \Cov_\mu(e_j(x)^2, e_\ell(x)) = \E_\mu[e_j^2 e_\ell] = 0
\end{equation}
unless $\ell = 2j$ (the product-to-sum selection rule). When $2j > d$, $\hat{v}_j \perp \bar{e}$ completely; for $2j \leq d$, partial information is captured through $\bar{e}_{2j}$ alone, so $\rho_j$ remains bounded away from $1$. This makes the general claim concrete in the canonical setting.

The off-diagonal terms $\hat{V}_{jk}$ for $j \neq k$ contribute additional $O(d^2/n)$ but, weighted by $\lambda_j \lambda_k$ which decays, are absorbed into $d_{\mathrm{eff}}/n$.
\end{proof}

\begin{remark}
Dimension counting confirms the structural deficit. Mean aggregation provides $d$ real numbers. The posterior covariance $\Sigma_p$ has $d(d+1)/2$ free parameters (a symmetric matrix). The deficit of $d(d-1)/2$ degrees of freedom is the source of the irreducible gap. Second-order aggregation (using $\frac{1}{n}\sum_i \phi(x_i)\phi(x_i)^\top$ as the representation) provides exactly $d(d+1)/2$ numbers, matching the posterior's degrees of freedom and eliminating the gap up to $O(d^2/n^2)$ estimation variance.
\end{remark}

\section{Combined Upper Bound}\label{sec:main}

The three error sources of \Cref{sec:variance_decomp} combine into an upper bound on the predictive KL divergence. We state the bound as a structural decomposition where each term traces to a specific architectural mechanism.

\begin{theorem}[Decomposition of the predictive KL]
\label{thm:main}
Under Assumptions~\ref{ass:bounded_var}--\ref{ass:tail}, for a latent neural process with mean aggregation, representation dimension $d$, and Mercer-feature encoder, trained on a GP with kernel $k$ and observation noise $\sigma_\epsilon^2$:
\begin{align}\label{eq:main_bound}
    &\E_C\!\left[\KL(p_{\mathrm{GP}} \| p_{\mathrm{LNP}})\right] \nonumber\\
    &\quad \leq \frac{3}{2\sigma_\ell^4}\Bigl(
        \underbrace{\Lambda^2 \kappa_k}_{\text{label, signal}}
      + \underbrace{\tfrac{\Lambda^2 \sigma_\epsilon^2}{n}}_{\text{label, noise}} \nonumber\\
    &\qquad\qquad
      + \underbrace{C_S\, \tau_d(k, x_*)}_{\text{bottleneck}}
      + \underbrace{R_d(k, n)}_{\text{estimation}}
    \Bigr) \nonumber\\
    &\quad\quad + \mathcal{A}(d, n),
\end{align}
where $\Lambda = L_\Sigma B_w^2 B_\psi$ collects the architectural constants from Assumptions~\ref{ass:encoder}--\ref{ass:lipschitz}, $C_S = C\kappa_k/\sigma_\epsilon^4$ is the truncation constant from \Cref{thm:bottleneck}, $\tau_d(k, x_*) = \sum_{j > d} \lambda_j e_j(x_*)^2$ is the kernel tail at the target, $R_d(k, n)$ is the estimation error from $\bar\phi_X$ as defined in \Cref{thm:bottleneck}, and $\mathcal{A}(d, n)$ is the amortization contribution to the KL.
\end{theorem}

\begin{proof}
The KL between two univariate Gaussians \eqref{eq:kl} satisfies
\begin{equation}
    \KL(p_{\mathrm{GP}} \| p_{\mathrm{LNP}}) \leq \frac{(\Delta \sigma^2)^2}{2 \sigma_\ell^4} + \frac{(\mu_{\mathrm{GP}} - \mu_{\mathrm{LNP}})^2}{2 \sigma_\ell^2},
\end{equation}
using $\sigma^2_{\mathrm{LNP}} \geq \sigma_\ell^2$ from \Cref{ass:bounded_var} and $|\log(1+x) - x| \leq x^2$ for small $x$. The variance error decomposes by the law of total variance over $\bm{y}$ given $X$:
\begin{align}
    \E[(\Delta\sigma^2)^2]
    &\leq 3\E\!\left[\Var_{\bm{y}|X}[\sigma^2_{\mathrm{LNP}}]\right] \nonumber\\
    &\quad + 3\E\!\left[(\sigma^2_{\mathrm{GP}}(x_*; X) - g_d^*(\bar\phi_X))^2\right] \nonumber\\
    &\quad + 3\E\!\left[(g_d^*(\bar\phi_X) - \E[\sigma^2_{\mathrm{LNP}} \mid X])^2\right].
\end{align}
The first term is bounded by \Cref{thm:label}, contributing $\Lambda^2(\kappa_k + \sigma_\epsilon^2/n)$. The second is bounded by \Cref{thm:bottleneck}, contributing $C_S\, \tau_d(k, x_*) + R_d(k, n)$. The third is the variance-pathway amortization contribution, which under mean aggregation is $\Theta(d_{\mathrm{eff}}/n)$ by \Cref{prop:amort_general}. Combining and bounding the mean error similarly yields~\eqref{eq:main_bound}; the mean contributions are absorbed into $\mathcal{A}(d, n)$.
\end{proof}

\subsection{Numerical illustration}\label{sec:numerics}

We illustrate the architectural implications of the bound numerically with four findings summarized below. The full methods, additional sample sizes, and the verification of \Cref{prop:amort_scalar}'s closed-form prediction (accurate to within $1\%$ for $n \geq 50$) appear in \Cref{app:numerics}. The amortization-gap experiments use the SE kernel on $[0,1]$ with $\ell = 0.3$ and the Mercer encoder of \Cref{ass:mercer_encoder}; the label contamination experiment uses a trained MLP-based LNP described in \Cref{app:label_exp}.

\paragraph{Mean vs.\ second-order aggregation.}
Table~\ref{tab:mean_vs_so_main} compares the empirical amortization gap under mean aggregation $r_C = \frac{1}{n}\sum_i \phi(x_i)$ against the second-order alternative $\frac{1}{n}\sum_i \phi(x_i)\phi(x_i)^\top$ at $d = 3$. Mean aggregation decays as $n^{-0.81}$, second-order as $n^{-1.75}$, matching the $\Theta(1/n)$ versus $O(1/n^2)$ separation predicted by \Cref{prop:amort_general}. The ratio between the two grows linearly in $n$, so the architectural choice is increasingly consequential at the context sizes typical of NP applications.

\begin{table}[h]
\centering
\small
\caption{Empirical amortization gap under mean vs.\ second-order aggregation ($d = 3$, SE kernel). Selected rows; full table in \Cref{app:numerics}.}\label{tab:mean_vs_so_main}
\begin{tabular}{rccc}
\toprule
$n$ & Mean agg.\ gap & 2nd-order gap & Ratio \\
\midrule
50  & $8.2 \times 10^{-3}$ & $1.6 \times 10^{-3}$  & $5\times$  \\
100 & $5.2 \times 10^{-3}$ & $5.3 \times 10^{-4}$  & $10\times$ \\
500 & $1.4 \times 10^{-3}$ & $2.8 \times 10^{-5}$  & $50\times$ \\
\bottomrule
\end{tabular}
\end{table}

\paragraph{Effective dimension.}
The amortization gap under mean aggregation depends on $d$ non-monotonically (Table~\ref{tab:dim_main}): it grows from $d = 1$ to $d = 3$ as more second-order structure becomes available to miss, then declines as eigenvalue decay reduces the sensitivity of $\Sigma_p$ to the unobserved features. By $d = 8$ the two aggregation schemes converge, since $\lambda_j$ for $j > 3$ is small enough under the SE kernel ($\ell = 0.3$) that the corresponding second-order terms contribute negligibly. This tracks the spectrum-weighted effective dimension $d_{\mathrm{eff}}$ rather than $d$ itself, consistent with the $\Theta(d_{\mathrm{eff}}/n)$ scaling of \Cref{prop:amort_general}. Under second-order aggregation the gap remains uniformly small across all $d$ tested.

\begin{table}[h]
\centering
\small
\caption{Amortization gap vs.\ representation dimension ($n = 50$, SE kernel).}\label{tab:dim_main}
\begin{tabular}{rcc}
\toprule
$d$ & Mean agg.\ gap & 2nd-order gap \\
\midrule
1 & $2.5 \times 10^{-3}$ & $5 \times 10^{-5}$  \\
3 & $8.3 \times 10^{-3}$ & $1.6 \times 10^{-3}$ \\
5 & $2.4 \times 10^{-3}$ & $1.6 \times 10^{-3}$ \\
8 & $1.7 \times 10^{-3}$ & $1.7 \times 10^{-3}$ \\
\bottomrule
\end{tabular}
\end{table}

\paragraph{The mechanism: identical representations, different posteriors.}
The mean-aggregation pathology has a concrete instantiation. For $d = 1$ on $[0,1]$ with $e_1(x) = \sqrt{2}\cos(\pi x)$, any pair $(x_1, x_2) = (a, 1-a)$ yields $\bar\phi = 0$ by symmetry. Yet the posterior precision $\Sigma_p^{-1} = 1 + \frac{\lambda_1}{\sigma_d^2}[e_1(x_1)^2 + e_1(x_2)^2]$ varies substantially across such pairs: $(0.10, 0.90)$ gives $\Sigma_p = 0.30$ (informative spread-out contexts), while $(0.45, 0.55)$ gives $\Sigma_p = 0.94$ (clustered, uninformative contexts). The amortized encoder must produce a single output for all pairs sharing $\bar\phi = 0$, incurring an average KL of $0.10$ between the extreme cases. This is the dimension-counting deficit of \Cref{prop:amort_general} ($d$ scalars summarizing a $d(d+1)/2$-parameter posterior covariance) made concrete in the smallest non-trivial case. Full pathology table in \Cref{app:numerics}.

\paragraph{Label contamination on a trained LNP.}
\Cref{thm:label}'s two-component prediction holds empirically on a trained MLP-based LNP that violates the encoder-affinity \Cref{ass:encoder}. Resampling labels at fixed context locations $X$, the variance of $\sigma^2_{\mathrm{LNP}}(x_*; C)$ separates into a signal-induced floor that plateaus at $\approx 0.21$ across $n \in [5, 1000]$ and a noise-induced component that decays from $1.5 \times 10^{-2}$ to $8 \times 10^{-4}$. The floor-to-noise ratio grows from $20\times$ at $n = 5$ to $256\times$ at $n = 1000$, demonstrating the regime change predicted by \Cref{thm:label}: at small $n$ both components contribute, while at large $n$ the irreducible $O(1)$ signal floor entirely dominates the contamination. The qualitative pattern survives the violation of \Cref{ass:encoder}, suggesting the structural mechanism (label dependence in the variance pathway) is general even when the precise constants are not. Full setup, results, and slope fits are in \Cref{app:label_exp}.

\paragraph{Mercer alignment of trained encoders.}
\Cref{ass:mercer_encoder} posits that the encoder recovers the top kernel eigenfunctions. We verify this by measuring the principal angles between the trained encoder's 64-dimensional feature subspace and the top-$d$ Mercer eigenfunction subspace across three SE kernels ($\ell \in \{0.1, 0.2, 0.5\}$). The encoder recovers each eigenfunction with $R^2 > 0.95$ whenever the corresponding eigenvalue exceeds the noise floor, and alignment degrades precisely where the spectrum becomes negligible: for $\ell = 0.5$ the subspace is near-perfectly aligned ($\cos\theta_{\min} > 0.95$) through $d = 7$, collapsing at $d = 9$ where $\lambda_j$ drops below $10^{-8}$. Shorter lengthscales produce slower spectral decay, and the encoder accordingly recovers more eigenfunctions before alignment degrades. No explicit spectral regularization is applied; the alignment arises because the ELBO training objective implicitly rewards capturing the directions of maximum variation under the GP prior. This suggests the bottleneck rates in \Cref{cor:rates} are approximately predictive of trained-NP behavior in this setup. Full methods and tables are in \Cref{app:mercer}.

\section{Discussion}
\label{sec:discussion}
 
We provided the quantitative decomposition of the gap between neural process and Gaussian process predictives. The three error sources, label contamination, information bottleneck, and amortization, have distinct origins and distinct architectural remedies.

The decomposition provides concrete guidance for neural process design. The bottleneck term governs the minimum $d$ needed for a given kernel. For SE kernels, $d = O(\sqrt{\log(1/\epsilon)})$ suffices for error $\epsilon$, meaning even small representations work well. For Mat\'ern-$\nu$ kernels with small $\nu$ (rough functions), $d$ must scale as $O(\epsilon^{-d_x/(2\nu)})$, motivating substantially larger representations. Whether the linear dependence on the tail sum can be improved to quadratic under distributional assumptions on the context locations remains open. Our bounds characterize pointwise predictive quality; a complementary question is whether NP predictives are consistent as stochastic processes, which \citet{young2026consistency} analyzes for CNPs.

The label contamination term arises because the LNP estimates variance through a label-dependent representation. The signal component is $O(1)$ and does not vanish with increasing context size, making this a persistent source of error rather than a finite-sample artifact. An architecture that computes variance from context locations alone, without access to labels, would eliminate this term entirely. This provides strong theoretical motivation for architectures that separate the mean and variance estimation pathways, predicting uncertainty from the input geometry alone.

The above architectural recommendation applies to the GP-amortization use case considered here, where data is drawn from a fixed known GP prior and exact GP inference is the ground truth. In meta-learning regimes where NPs are trained across distributions of related tasks (or across GPs with varying hyperparameters), the same label-dependence we identify as ``contamination'' may carry Bayes-relevant information about task structure, and a location-only variance head may underperform. Characterizing this tradeoff requires bounding the approximation gap under varying or misspecified priors, which we leave to future work.

Attentive neural processes replace mean aggregation with cross-attention. The bottleneck analysis extends since attention weights can approximate the GP kernel weights $\alpha_i(x_*) \propto k(x_*, x_i)$, potentially reducing the mean approximation error. However, the variance still flows through a finite-dimensional latent, so the bottleneck bound $S_d$ for the variance pathway is unchanged. Our analysis suggests attention primarily addresses the mean pathway, and whether attention can implicitly capture the second-order context statistics relevant to the variance remains open.

The costs identified above should be weighed against the per-task optimization cost that amortization eliminates. Sparse variational GPs require jointly optimizing inducing locations, kernel hyperparameters, and variational parameters for each new context set, which is a non-convex problem typically solved by gradient descent on the ELBO. In sequential settings where the context changes at each iteration, or in real-time settings where predictions are needed in milliseconds, this per-task optimization may dominate the approximation quality gap. The bounds we derive characterize what is lost in exchange for reducing inference to a single forward pass, and whether this tradeoff is favorable depends on the deployment regime. In large-scale spatial or environmental monitoring tasks for example, where GP-quality predictions are needed across thousands or millions of related subdomains, per-task sparse GP optimization is infeasible regardless of the per-task cost, and amortization is the only viable path to posterior inference at scale.

The bottleneck result connects neural process architecture to kernel theory. The representation dimension $d$ must be matched to the smoothness of the function class, as measured by the eigenvalue decay of the kernel. This parallels both classical results on $n$-widths in approximation theory \citep{pinkus1985nwidths} and convergence rates for sparse variational GPs \citep{burt2020convergence}, where the number of inducing points required for a given approximation quality is governed by the same eigenvalue decay. The neural process bottleneck bound can be viewed as an amortized analogue of the sparse GP result, with the additional label contamination and amortization terms reflecting costs specific to the learned encoder architecture.

Our analysis assumes the neural process is trained on data drawn from a known GP prior, which is the setting where exact GP inference is available. The practical value of the bounds lies in their structural character. The three error sources arise from architectural choices (mean aggregation, finite bottleneck, amortized encoder) that are present regardless of the data-generating process. We conjecture that the bottleneck smoothness correspondence extends to misspecified settings where the true function has regularity comparable to a given kernel class, but formalizing this requires bounding the approximation gap under model misspecification, which we leave to future work.

\Cref{ass:bounded_var} is mild as the GP variance is automatically bounded for any positive-definite kernel with bounded diagonal and positive noise, and the LNP variance is bounded whenever the decoder has a positive floor $\sigma_d^2 > 0$ and bounded weights. In practice, all standard neural process implementations satisfy these conditions.

Several directions remain open. Our bounds are upper bounds and matching lower bounds would establish whether the rates are tight. Whether attention-based aggregation can implicitly capture second-order statistics, thereby closing the amortization gap identified earlier, deserves investigation. Finally, extending the analysis to convolutional neural processes \citep{Gordon2020Convolutional}, which exploit translation equivariance, could yield tighter bottleneck bounds when the kernel is stationary.



\bibliography{references}

@InProceedings{garnelo2018cnp,
  title = 	 {Conditional Neural Processes},
  author =       {Garnelo, Marta and Rosenbaum, Dan and Maddison, Christopher and Ramalho, Tiago and Saxton, David and Shanahan, Murray and Teh, Yee Whye and Rezende, Danilo and Eslami, S. M. Ali},
  booktitle = 	 {Proceedings of the 35th International Conference on Machine Learning},
  pages = 	 {1704--1713},
  year = 	 {2018},
  editor = 	 {Dy, Jennifer and Krause, Andreas},
  volume = 	 {80},
  series = 	 {Proceedings of Machine Learning Research},
  month = 	 {10--15 Jul},
  publisher =    {PMLR},
  url = 	 {https://proceedings.mlr.press/v80/garnelo18a.html},
}

@misc{garnelo2018neuralprocesses,
      title={Neural Processes}, 
      author={Marta Garnelo and Jonathan Schwarz and Dan Rosenbaum and Fabio Viola and Danilo J. Rezende and S. M. Ali Eslami and Yee Whye Teh},
      year={2018},
      eprint={1807.01622},
      archivePrefix={arXiv},
      primaryClass={cs.LG},
      url={https://arxiv.org/abs/1807.01622}, 
}

@inproceedings{kim2018attentive,
title={Attentive Neural Processes},
author={Hyunjik Kim and Andriy Mnih and Jonathan Schwarz and Marta Garnelo and Ali Eslami and Dan Rosenbaum and Oriol Vinyals and Yee Whye Teh},
booktitle={International Conference on Learning Representations},
year={2019},
url={https://openreview.net/forum?id=SkE6PjC9KX},
}

@article{young2026consistency,
title={On the Conditioning Consistency Gap in Conditional Neural Processes},
author={Robin Young},
journal={Transactions on Machine Learning Research},
issn={2835-8856},
year={2026},
url={https://openreview.net/forum?id=rLJ5Hm5vbG},
note={}
}

@book{rasmussen2005gaussian,
  title     = {Gaussian Processes for Machine Learning},
  author    = {Rasmussen, Carl Edward and Williams, Christopher K. I.},
  year      = {2005},
  publisher = {The MIT Press},
  isbn      = {9780262256834},
  doi       = {10.7551/mitpress/3206.001.0001},
  url       = {https://doi.org/10.7551/mitpress/3206.001.0001}
}

@inproceedings{Gordon2020Convolutional,
title={Convolutional Conditional Neural Processes},
author={Jonathan Gordon and Wessel P. Bruinsma and Andrew Y. K. Foong and James Requeima and Yann Dubois and Richard E. Turner},
booktitle={International Conference on Learning Representations},
year={2020},
url={https://openreview.net/forum?id=Skey4eBYPS}
}

@inproceedings{foong2020meta,
 author = {Foong, Andrew and Bruinsma, Wessel and Gordon, Jonathan and Dubois, Yann and Requeima, James and Turner, Richard},
 booktitle = {Advances in Neural Information Processing Systems},
 editor = {H. Larochelle and M. Ranzato and R. Hadsell and M.F. Balcan and H. Lin},
 pages = {8284--8295},
 publisher = {Curran Associates, Inc.},
 title = {Meta-Learning Stationary Stochastic Process Prediction with Convolutional Neural Processes},
 url = {https://proceedings.neurips.cc/paper_files/paper/2020/file/5df0385cba256a135be596dbe28fa7aa-Paper.pdf},
 volume = {33},
 year = {2020}
}

@inproceedings{bruinsma2023autoregressive,
title={Autoregressive Conditional Neural Processes},
author={Wessel Bruinsma and Stratis Markou and James Requeima and Andrew Y. K. Foong and Tom Andersson and Anna Vaughan and Anthony Buonomo and Scott Hosking and Richard E Turner},
booktitle={The Eleventh International Conference on Learning Representations },
year={2023},
url={https://openreview.net/forum?id=OAsXFPBfTBh}
}

@InProceedings{cremer2018inference,
  title = 	 {Inference Suboptimality in Variational Autoencoders},
  author =       {Cremer, Chris and Li, Xuechen and Duvenaud, David},
  booktitle = 	 {Proceedings of the 35th International Conference on Machine Learning},
  pages = 	 {1078--1086},
  year = 	 {2018},
  editor = 	 {Dy, Jennifer and Krause, Andreas},
  volume = 	 {80},
  series = 	 {Proceedings of Machine Learning Research},
  month = 	 {10--15 Jul},
  publisher =    {PMLR},
  url = 	 {https://proceedings.mlr.press/v80/cremer18a.html},
}

@book{wendland2004scattered,
  title     = {Scattered Data Approximation},
  author    = {Wendland, Holger},
  series    = {Cambridge Monographs on Applied and Computational Mathematics},
  volume    = {17},
  publisher = {Cambridge University Press},
  year      = {2004},
  isbn      = {9781139456654},
}

@book{pinkus1985nwidths,
  title     = {n-Widths in Approximation Theory},
  author    = {Pinkus, Allan},
  publisher = {Springer Berlin Heidelberg},
  year      = {1985},
  doi       = {10.1007/978-3-642-69894-1},
  isbn      = {978-3-642-69894-1},
  issn      = {0071-1136},
}

@article{SteinwartScovel2012,
  author    = {Ingo Steinwart and Clint Scovel},
  title     = {{Mercer's Theorem on General Domains: On the Interaction between Measures, Kernels, and {RKHS}s}},
  journal   = {Constructive Approximation},
  volume    = {35},
  number    = {3},
  pages     = {363--417},
  year      = {2012},
  doi       = {10.1007/s00365-012-9153-3},
  publisher = {Springer}
}

@article{burt2020convergence,
  author  = {David R. Burt and Carl Edward Rasmussen and Mark van der Wilk},
  title   = {{Convergence of Sparse Variational Inference in Gaussian Processes Regression}},
  journal = {Journal of Machine Learning Research},
  year    = {2020},
  volume  = {21},
  number  = {131},
  pages   = {1--63},
  url     = {http://jmlr.org/papers/v21/19-1015.html}
}

@InProceedings{titsias2009sparse,
  title = 	 {{Variational Learning of Inducing Variables in Sparse Gaussian Processes}},
  author = 	 {Titsias, Michalis},
  booktitle = 	 {Proceedings of the Twelfth International Conference on Artificial Intelligence and Statistics},
  pages = 	 {567--574},
  year = 	 {2009},
  editor = 	 {van Dyk, David and Welling, Max},
  volume = 	 {5},
  series = 	 {Proceedings of Machine Learning Research},
  address = 	 {Hilton Clearwater Beach Resort, Clearwater Beach, Florida USA},
  month = 	 {16--18 Apr},
  publisher =    {PMLR},
  url = 	 {https://proceedings.mlr.press/v5/titsias09a.html},
}

@article{cranmer2020frontier,
  title   = {The frontier of simulation-based inference},
  author  = {Cranmer, Kyle and Brehmer, Johann and Louppe, Gilles},
  journal = {Proceedings of the National Academy of Sciences},
  volume  = {117},
  number  = {48},
  pages   = {30055--30062},
  year    = {2020},
  month   = {12},
  doi     = {10.1073/pnas.1912789117},
  url     = {https://doi.org/10.1073/pnas.1912789117}
}

@ARTICLE{deisenroth2015gaussian,
author={Deisenroth, Marc Peter and Fox, Dieter and Rasmussen, Carl Edward},
journal={ IEEE Transactions on Pattern Analysis \& Machine Intelligence },
title={{ Gaussian Processes for Data-Efficient Learning in Robotics and Control }},
year={2015},
volume={37},
number={02},
ISSN={1939-3539},
pages={408-423},
doi={10.1109/TPAMI.2013.218},
url = {https://doi.ieeecomputersociety.org/10.1109/TPAMI.2013.218},
publisher={IEEE Computer Society},
address={Los Alamitos, CA, USA},
month=feb}

@article{shahriari2016taking,
  title   = {Taking the Human Out of the Loop: A Review of Bayesian Optimization},
  author  = {Shahriari, Bobak and Swersky, Kevin and Wang, Ziyu and Adams, Ryan P. and de Freitas, Nando},
  journal = {Proceedings of the IEEE},
  volume  = {104},
  number  = {1},
  pages   = {148--175},
  year    = {2016},
  month   = {01},
  doi     = {10.1109/JPROC.2015.2494218},
  url     = {https://doi.org/10.1109/JPROC.2015.2494218}
}

\appendix
\onecolumn
\crefalias{section}{appendix}
\crefalias{subsection}{appendix}

\section{Primer on Neural Processes}\label{app:np_primer}

This appendix provides a brief background on neural processes for readers more familiar with Gaussian processes than with the NP literature, and discusses how the assumptions used in our analysis relate to NP architectures used in practice.

\subsection{Architecture and forward pass}

A latent neural process \citep{garnelo2018neuralprocesses} maps a context set $C = \{(x_i, y_i)\}_{i=1}^n$ and a target location $x_*$ to a predictive distribution over $y_*$ in $O(n)$ time. Figure~\ref{fig:np_schematic} shows the forward-pass pipeline annotated with the three error sources of \Cref{sec:variance_decomp}. For the full graphical model view that also distinguishes generative from inference paths, see Figure~1 from \citet{garnelo2018neuralprocesses}.

\begin{figure}[h]
\centering
\begin{tikzpicture}[
    node distance=5mm and 7mm,
    box/.style={draw, rounded corners=2pt, minimum height=7mm, minimum width=14mm, font=\small, inner sep=2pt},
    inp/.style={font=\small, inner sep=1pt},
    arrow/.style={-{Latex[length=2mm]}, thick},
    err/.style={font=\footnotesize\itshape, text=red!70!black, align=center},
    ctxbox/.style={draw, dashed, rounded corners=3pt, inner sep=4pt, line width=0.4pt}
  ]
  \node[inp] (xC) {$\{x_i\}$};
  \node[inp, below=2mm of xC] (yC) {$\{y_i\}$};
  \node[box, right=of $(xC)!0.5!(yC)$] (h) {$h(x_i, y_i)$};
  \node[box, right=of h] (agg) {$r_C$};
  \node[box, right=of agg] (rec) {$\mu_z, \Sigma_z$};
  \node[box, right=of rec] (z) {$z$};
  \node[box, right=8mm of z] (dec) {decoder};
  \node[inp, above=2mm of dec] (xstar) {$x_*$};
  \node[box, right=of dec] (pred) {$p(y_* \mid x_*, z)$};

  \draw[arrow] (xC.east) -- (h.west);
  \draw[arrow] (yC.east) -- (h.west);
  \draw[arrow] (h) -- (agg);
  \draw[arrow] (agg) -- (rec);
  \draw[arrow] (rec) -- (z);
  \draw[arrow] (z) -- (dec);
  \draw[arrow] (xstar) -- (dec);
  \draw[arrow] (dec) -- (pred);

  \begin{scope}[on background layer]
    \node[ctxbox, fit=(xC)(yC)(h)(agg)(rec)(z), label={[font=\footnotesize\itshape, text=gray!60!black]above:context pipeline (operates on $n$ pairs)}] {};
  \end{scope}

  \node[err, below=4mm of h] {label\\ contamination};
  \node[err, below=4mm of agg] {information\\ bottleneck};
  \node[err, below=4mm of rec] {amortization\\ gap};
\end{tikzpicture}
\caption{Forward-pass schematic of a latent neural process with the three error sources annotated. The dashed box encloses the context pipeline. The encoder $h$ is applied independently to each of the $n$ context pairs $(x_i, y_i)$ producing per-point features that mix $x$ and $y$ inputs (label contamination, \Cref{sec:label_cont}), mean aggregation $r_C = \frac{1}{n}\sum_i h(x_i, y_i)$ compresses these to a finite-dimensional summary (information bottleneck, \Cref{sec:bottleneck}), the recognition network $(\mu_z, \Sigma_z)$ maps $r_C$ to a latent posterior using a single map shared across all contexts (amortization gap, \Cref{sec:amort}). The target location $x_*$ enters only at the decoder.}
\label{fig:np_schematic}
\end{figure}
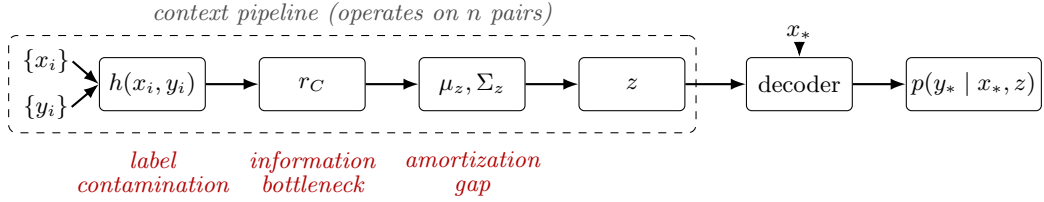

The components are typically parameterized as follows:

\begin{itemize}
    \item \textbf{Encoder} $h: \mathcal{X} \times \mathcal{Y} \to \R^{d}$ is an MLP applied independently to each context pair $(x_i, y_i)$.
    \item \textbf{Aggregation} reduces the variable-size set $\{h_i\}_{i=1}^n$ to a fixed-size representation $r_C \in \R^d$. Mean aggregation $r_C = \frac{1}{n}\sum_i h_i$ is the original choice \citep{garnelo2018cnp, garnelo2018neuralprocesses}. Cross-attention \citep{kim2018attentive} and convolutional pooling \citep{Gordon2020Convolutional} are common alternatives.
    \item \textbf{Recognition network} $(\mu_z, \Sigma_z): \R^d \to \R^{d_z} \times \R^{d_z \times d_z}_{\succ 0}$ maps $r_C$ to the parameters of a Gaussian latent distribution. The covariance is typically diagonal or low-rank, with positive definiteness enforced via a softplus or Cholesky parameterization.
    \item \textbf{Decoder} $g_\theta: \mathcal{X} \times \R^{d_z} \to \R \times \R_{>0}$ produces a predictive mean and variance at each target. In practice $g_\theta$ is a nonlinear MLP; our linear decoder $w(x_*)^\top z + b(x_*)$ in \Cref{def:lnp} is a tractability choice.
\end{itemize}

The marginal predictive at $x_*$ is obtained by integrating out $z$:
\begin{equation}
    p_{\mathrm{LNP}}(y_* \mid x_*, C) = \int p(y_* \mid x_*, z)\, q(z \mid C)\, dz.
\end{equation}
Under our linear decoder this integral is closed-form Gaussian (equations~\eqref{eq:lnp_mean}--\eqref{eq:lnp_var}); under nonlinear decoders it is typically estimated by Monte Carlo over $z$.

\subsection{Training objective}

NPs are trained by maximizing an evidence lower bound (ELBO) over context/target splits sampled from a task distribution. Given a task with combined data $D = \{(x_i, y_i)\}_{i=1}^N$, randomly partitioned into context $C$ and target $T$:
\begin{equation}
    \mathcal{L}(\theta; D) = \E_{q(z \mid C \cup T)}\!\left[\log p(y_T \mid x_T, z)\right] - \KL\bigl(q(z \mid C \cup T)\,\|\,q(z \mid C)\bigr).
\end{equation}
The first term encourages accurate target predictions and the second regularizes the context-only encoder toward the full-data encoder. Tasks may be drawn from a fixed GP prior (the regime our analysis targets) or from a heterogeneous task distribution (the meta-learning regime discussed in \Cref{sec:discussion}).

\subsection{Relationship of the assumptions to practice}
\label{app:assumptions}

Our analysis makes three substantive idealizations of the LNP architecture beyond the basic \Cref{def:lnp}: the linear-in-$z$ decoder (item 4 of the definition), the affine-in-$y$ encoder (\Cref{ass:encoder}), and the Mercer-feature encoder (\Cref{ass:mercer_encoder}). The bounded-variance \Cref{ass:bounded_var}, the Lipschitz \Cref{ass:lipschitz}, and the tail \Cref{ass:tail} are technical conditions that hold under standard NP implementations and reasonable regularity. We discuss the realism of each below.

\paragraph{Linear decoder (\Cref{def:lnp}, item 4).} Real NPs use nonlinear MLP decoders $g_\theta(x_*, z)$. The linear form $w(x_*)^\top z + b(x_*)$ is a tractability choice with two consequences: it makes the marginal predictive Gaussian in closed form (equations~\eqref{eq:lnp_mean}--\eqref{eq:lnp_var}), and it makes the posterior $p(z \mid C)$ Gaussian with closed-form precision $\Sigma_p^{-1} = I + \sigma_d^{-2}\sum_i w(x_i) w(x_i)^\top$ (equation~\eqref{eq:posterior}). The latter is what makes the amortization gap characterizable in Propositions~\ref{prop:amort_scalar}--\ref{prop:amort_general}; without it, $p(z \mid C)$ is intractable and the gap cannot be expressed in closed form.

The dimension-counting argument that motivates the second-order aggregation recommendation, that mean aggregation provides $d$ scalars while $\Sigma_p$ has $d(d+1)/2$ free parameters, is structural and does not depend on the linear decoder. We expect the qualitative recommendation to transfer to nonlinear decoders, though the constants do not.

\paragraph{Encoder affine in $y$ (\Cref{ass:encoder}).} Real NPs use general MLPs $h(x, y)$ that are non-affine in $y$. The affine assumption $h(x, y) = \phi(x) + \psi(x) y$ allows the clean decomposition $r_C = \bar\phi_X + \delta_y$ in equation~\eqref{eq:rep_decomp}, isolating the label-dependent component $\delta_y$. The qualitative conclusion of \Cref{thm:label} is a structural property of any encoder that mixes $x$ and $y$ inputs into a representation that determines the predictive variance. The empirical verification in \Cref{app:label_exp} uses an MLP encoder that is non-affine in $y$ and observes the predicted two-component decomposition, suggesting the qualitative result transfers. The constants in \Cref{thm:label} are specific to the affine setting.

\paragraph{Mercer-feature encoder (\Cref{ass:mercer_encoder}).} Real NPs do not necessarily learn Mercer eigenfunctions as the standard ELBO contains no term that would explicitly encourage this alignment. \Cref{ass:mercer_encoder} should be read as characterizing the best-case spectral alignment. If the encoder achieves the optimal $d$-dimensional kernel approximation, then the bottleneck rates in \Cref{cor:rates} apply. A learned encoder achieving suboptimal alignment would incur additional approximation error in $\bar\phi_X$ relative to the eigenfunction projection, beyond what \Cref{thm:bottleneck} characterizes.

We do not bound this gap analytically. However, the empirical verification in \Cref{app:mercer} shows that trained NP encoders do approximately recover the top Mercer eigenfunctions, with $R^2 > 0.95$ for all eigenfunctions whose eigenvalues exceed the noise floor. The alignment degrades where the spectrum becomes negligible, and the effective aligned dimension tracks the kernel lengthscale. This suggests the bottleneck rates in \Cref{cor:rates} are approximately predictive of trained-NP behavior in the GP-amortization regime, though the formal gap between the Mercer encoder and a learned encoder remains an open problem.

\subsection{NP variants and which assumptions they satisfy}

\begin{itemize}
    \item \textbf{Conditional NPs} \citep{garnelo2018cnp}: no latent $z$; deterministic context-to-prediction map. Our amortization-gap analysis does not apply but the bottleneck and label-contamination analyses do.
    \item \textbf{Latent NPs} \citep{garnelo2018neuralprocesses}: the architecture our analysis targets, modulo the linear-decoder and Mercer-feature idealizations.
    \item \textbf{Attentive NPs} \citep{kim2018attentive}: replace mean aggregation with cross-attention. The bottleneck bound on the variance pathway is unchanged (variance still flows through a finite-dimensional latent). Whether attention can implicitly capture second-order context statistics relevant to the variance is open.
    \item \textbf{Convolutional NPs} \citep{Gordon2020Convolutional}: exploit translation equivariance for stationary kernels. Our bottleneck analysis does not directly apply, though we conjecture the spectral correspondence extends with modified rates.
\end{itemize}

\section{Proofs}\label{app:proofs}

This appendix contains the full proofs of Theorems~\ref{thm:label} and~\ref{thm:bottleneck} and \Cref{prop:amort_scalar}.

\subsection{Proof of \Cref{thm:label} (Label Contamination Bound)}
\label{app:proof_label}

\begin{proof}
Conditional on $X$, the label-dependent component $\delta_y = \frac{1}{n}\sum_{i=1}^n \psi(x_i) y_i$ has:
\begin{equation}
    \E[\delta_y \mid X] = \frac{1}{n}\sum_i \psi(x_i)\, \E[y_i] = 0
\end{equation}
under the zero-mean GP prior. The covariance is:
\begin{align}
    \E[\delta_y \delta_y^\top \mid X] &= \frac{1}{n^2}\sum_{i,j} \psi(x_i)\psi(x_j)^\top \E[y_i y_j] \\
    &= \frac{1}{n^2}\sum_{i,j} \psi(x_i)\psi(x_j)^\top (k(x_i, x_j) + \sigma_\epsilon^2 \delta_{ij}).
\end{align}
Decompose $y_i = f(x_i) + \epsilon_i$ to split $\delta_y = \delta_f + \delta_\epsilon$ where
\begin{equation}
    \delta_f = \frac{1}{n}\sum_i \psi(x_i) f(x_i), \qquad
    \delta_\epsilon = \frac{1}{n}\sum_i \psi(x_i) \epsilon_i.
\end{equation}
Since the $\epsilon_i$ are i.i.d.\ with variance $\sigma_\epsilon^2$ and independent of $f$:
\begin{equation}
    \E[\|\delta_\epsilon\|^2 \mid X]
    = \frac{\sigma_\epsilon^2}{n^2}\sum_i \|\psi(x_i)\|^2
    \leq \frac{B_\psi^2 \sigma_\epsilon^2}{n}.
\end{equation}
For the signal term, the covariance is
\begin{equation}
    \E[\delta_f \delta_f^\top \mid X]
    = \frac{1}{n^2}\sum_{i,j} \psi(x_i)\psi(x_j)^\top k(x_i, x_j).
\end{equation}
Taking the trace and using $\|\psi(x)\| \leq B_\psi$:
\begin{align}
    \E[\|\delta_f\|^2 \mid X]
    &\leq \frac{B_\psi^2}{n^2}\sum_{i,j}|k(x_i,x_j)| \\
    &\leq \frac{B_\psi^2}{n} \sup_i \sum_{j=1}^n \frac{|k(x_i,x_j)|}{n}
    \leq B_\psi^2 \kappa_k,
\end{align}
the last inequality using $\frac{1}{n}\sum_j |k(x_i,x_j)| \leq \sup_x k(x,x) = \kappa_k$, which holds since $k$ is positive definite with bounded diagonal. Note that $\|K\|_{\mathrm{op}} = O(n)$ for i.i.d.\ context points drawn from $\mu$, so this $O(1)$ bound cannot be improved in general.

Since $\delta_f$ and $\delta_\epsilon$ are independent conditional on $X$ (the noise $\epsilon_i$ is independent of $f$), the cross term vanishes: $\E[\delta_f^\top \delta_\epsilon \mid X] = \E[\delta_f \mid X]^\top \E[\delta_\epsilon \mid X] = 0$. Therefore:
\begin{equation}
    \E[\|\delta_y\|^2 \mid X] = \E[\|\delta_f\|^2 \mid X] + \E[\|\delta_\epsilon\|^2 \mid X] \leq B_\psi^2\!\left(\kappa_k + \frac{\sigma_\epsilon^2}{n}\right).
\end{equation}

Since $\Sigma_z$ is $L_\Sigma$-Lipschitz and the LNP variance is $\sigma^2_{\mathrm{LNP}} = w(x_*)^\top \Sigma_z(r_C) w(x_*) + \sigma_d^2$:
\begin{equation}
    |\sigma^2_{\mathrm{LNP}}(r_C) - \sigma^2_{\mathrm{LNP}}(\bar{\phi}_X)| \leq L_\Sigma B_w^2 \|\delta_y\|.
\end{equation}
Taking the conditional variance:
\begin{align}
    \Var_{\bm{y}|X}[\sigma^2_{\mathrm{LNP}}]
    &\leq \E_{\bm{y}|X}\bigl[(\sigma^2_{\mathrm{LNP}}(r_C) - \sigma^2_{\mathrm{LNP}}(\bar{\phi}_X))^2\bigr] \nonumber\\
    &\leq L_\Sigma^2 B_w^4\, \E\bigl[\|\delta_y\|^2 \mid X\bigr].
\end{align}
Taking the expectation over $X$ and substituting the bound on $\E[\|\delta_y\|^2 \mid X]$ yields the result.
\end{proof}

\subsection{Proof of \Cref{thm:bottleneck} (Information Bottleneck Bound)}\label{app:proof_bottleneck}

\begin{proof}
Write the GP variance using the Mercer expansion. The kernel vector and matrix satisfy
\begin{align}
    [\bm{k}_*]_i &= \sum_{j=1}^\infty \lambda_j\, e_j(x_*)\, e_j(x_i), \\
    [K]_{il} &= \sum_{j=1}^\infty \lambda_j\, e_j(x_i)\, e_j(x_l).
\end{align}
Define the empirical kernel feature $\hat{v}_j = \frac{1}{n}\sum_{i=1}^n e_j(x_i)^2$. For fixed context locations, the GP variance can be written via the Woodbury identity as a function of the spectrum and the empirical features $\{\hat{v}_j\}_{j=1}^\infty$. Note that $\bar{\phi}_X$ is a $d$-dimensional vector encoding $\frac{1}{n}\sum_i \sqrt{\lambda_j}\, e_j(x_i)$ for $j = 1, \dots, d$, which does not determine the diagonal empirical features $\hat{v}_j = \frac{1}{n}\sum_i e_j(x_i)^2$ for $j \leq d$. The optimal estimator $g_d^*(\bar{\phi}_X)$ is the best predictor of $\sigma^2_{\mathrm{GP}}$ given this limited statistic.

\paragraph{Setup: head/tail decomposition.}
Write the Mercer series as head ($j \leq d$) and tail ($j > d$) components. Define the head and tail prior variances
\begin{equation}
    k_H(x_*) = \sum_{j \leq d} \lambda_j\, e_j(x_*)^2, \qquad k_T(x_*) = \sum_{j > d} \lambda_j\, e_j(x_*)^2,
\end{equation}
the corresponding splits $\mathbf{k}_*^H$, $\mathbf{k}_*^T$ of the kernel vector and $K_H$, $K_T$ of the kernel matrix, and let
\begin{equation}
    A = K_H + \sigma_\epsilon^2 I, \qquad M = A + K_T.
\end{equation}
The truncated GP variance is then
\begin{equation}
    \sigma_d^2 = k_H - (\mathbf{k}_*^H)^\top A^{-1}\, \mathbf{k}_*^H,
\end{equation}
which depends only on head eigenfunction evaluations $\{e_j(x_i)\}_{j \leq d}$ at the context points. Since the optimal estimator $g_d^*(\bar{\phi}_X)$ minimizes MSE over all measurable functions of $\bar{\phi}_X$, and $\sigma_d^2$ is measurable with respect to a finer $\sigma$-algebra (the full head evaluations), we have
\begin{equation}
    S_d \leq \E_X\!\bigl[(\sigma^2_{\mathrm{GP}} - \sigma_d^2)^2\bigr]. \label{eq:Sd_reduce}
\end{equation}

\paragraph{Bounding $|\sigma^2_{\mathrm{GP}} - \sigma_d^2|$ via Neumann expansion.}
The full GP variance is $\sigma^2_{\mathrm{GP}} = (k_H + k_T) - \bm{k}_*^\top M^{-1} \bm{k}_*$. Writing $\Delta I = \bm{k}_*^\top M^{-1}\bm{k}_* - (\bm{k}_*^H)^\top A^{-1}\bm{k}_*^H$, we have
\begin{equation}
    \sigma^2_{\mathrm{GP}} - \sigma_d^2 = k_T - \Delta I. \label{eq:trunc_diff}
\end{equation}

We bound $\Delta I$ using a Neumann expansion. Pass to the whitened coordinates $\bm{a} = A^{-1/2}\bm{k}_*^H$ and $\bm{b} = A^{-1/2}\bm{k}_*^T$, and define $P = A^{-1/2}K_T A^{-1/2}$. Then $\|\bm{a}\|^2 = (\bm{k}_*^H)^\top A^{-1}\bm{k}_*^H$ and $\|\bm{b}\|^2 =: \beta = (\bm{k}_*^T)^\top A^{-1}\bm{k}_*^T$. Since $A \succeq \sigma_\epsilon^2 I$:
\begin{equation}\label{eq:alpha_beta_bounds}
    \|\bm{a}\|^2 \leq k_H / \sigma_\epsilon^2, \qquad \beta \leq k_T / \sigma_\epsilon^2.
\end{equation}
Now $M^{-1} = A^{-1/2}(I+P)^{-1}A^{-1/2}$, and under \Cref{ass:tail} ($\eta = \|P\|_{\mathrm{op}} < 1$):
\begin{equation}
    (I+P)^{-1} = I - P + Q, \qquad Q = \sum_{m=2}^\infty (-P)^m, \qquad \|Q\|_{\mathrm{op}} \leq \frac{\eta^2}{1 - \eta}.
\end{equation}
Expanding $\Delta I$ in these coordinates:
\begin{align}
    \Delta I &= (\mathbf{a} + \mathbf{b})^\top (I + P)^{-1}(\mathbf{a} + \mathbf{b}) - \|\mathbf{a}\|^2 \nonumber \\
    &= \underbrace{2\mathbf{a}^\top \mathbf{b} + \|\mathbf{b}\|^2}_{T_1} - \underbrace{\mathbf{a}^\top P\,\mathbf{a} + 2\,\mathbf{a}^\top P\,\mathbf{b} + \mathbf{b}^\top P\,\mathbf{b}}_{T_2} + \underbrace{(\mathbf{a}+\mathbf{b})^\top Q\,(\mathbf{a}+\mathbf{b})}_{T_3}.
\end{align}

By Cauchy--Schwarz:
\begin{align}
    |T_1| &\leq 2\|\bm{a}\|\sqrt{\beta} + \beta, \\
    |T_2| &\leq \eta\bigl(\|\bm{a}\| + \sqrt{\beta}\bigr)^2, \\
    |T_3| &\leq \frac{\eta^2}{1-\eta}\bigl(\|\bm{a}\| + \sqrt{\beta}\bigr)^2.
\end{align}
Substituting into~\eqref{eq:trunc_diff}:
\begin{align}
    |\sigma^2_{\mathrm{GP}} - \sigma_d^2| &= |k_T - \Delta I| \nonumber \\
    &\leq k_T + 2\|\mathbf{a}\|\sqrt{\beta} + \beta + \frac{\eta}{1-\eta}\bigl(\|\mathbf{a}\| + \sqrt{\beta}\bigr)^2.
\end{align}
Substituting the bounds~\eqref{eq:alpha_beta_bounds}:
\begin{equation}\label{eq:combined_bound}
    |\sigma^2_{\mathrm{GP}} - \sigma_d^2| \leq k_T + \frac{2\sqrt{\kappa_k\, k_T}}{\sigma_\epsilon^2} + \frac{k_T}{\sigma_\epsilon^2} + \frac{\eta}{1-\eta}\cdot\frac{(\sqrt{\kappa_k} + \sqrt{k_T})^2}{\sigma_\epsilon^2}.
\end{equation}

\paragraph{Extracting the rate.}
Every term on the right of~\eqref{eq:combined_bound} is $O(\sqrt{k_T})$ as $d \to \infty$ (using $\eta \leq \|K_T\|_{\mathrm{op}}/\sigma_\epsilon^2 \to 0$ and $k_T \to 0$). After squaring, the dominant contribution is:
\begin{equation}
    \left(\frac{2\sqrt{\kappa_k\, k_T}}{\sigma_\epsilon^2}\right)^2 = \frac{4\kappa_k}{\sigma_\epsilon^4}\, k_T,
\end{equation}
which is $O(k_T)$. All other squared terms are $O(k_T^{3/2})$ or smaller. Specifically, expanding~\eqref{eq:combined_bound} squared and using $ab \leq \tfrac{1}{2}(a^2 + b^2)$ to absorb cross terms:
\begin{equation}
    (\sigma^2_{\mathrm{GP}} - \sigma_d^2)^2 \leq C_S\, k_T(x_*),
\end{equation}
where, in the regime $\eta \leq 1/2$ and $k_T \leq \kappa_k$:
\begin{equation}
    C_S = \frac{C\,\kappa_k}{\sigma_\epsilon^4}
\end{equation}
for a constant $C$ depending only on $\kappa_k / \sigma_\epsilon^2$. Applying~\eqref{eq:Sd_reduce}:
\begin{equation}
    S_d(k, x_*) \leq C_S \sum_{j>d}\lambda_j\, e_j(x_*)^2.
\end{equation}

\paragraph{Estimation error $R_d$.}
The estimation error $R_d(k,n)$ arises from the variance of $\bar{\phi}_X$ as an estimator of $\E_\mu[\phi(x)]$ and is $O(d^2/n)$ by the central limit theorem applied to the $d$-dimensional i.i.d.\ sum, plus the irreducible information loss $R_d^{\mathrm{info}}$ from mean aggregation being an insufficient statistic for the head kernel matrix. We do not bound $R_d^{\mathrm{info}}$ in general as it depends on the relationship between the first and second moments of the eigenfunction evaluations under $\mu$. Under second-order aggregation, $\bar{\phi}_X$ is replaced by $\frac{1}{n}\sum_i \phi(x_i)\phi(x_i)^\top$, which is a sufficient statistic for the head kernel matrix, so $R_d^{\mathrm{info}} = 0$ and the estimation error reduces to $O(d^2/n)$.
\end{proof}

\subsection{Proof of \Cref{prop:amort_scalar} (Scalar Amortization Gap)}\label{app:proof_amort_scalar}

\begin{proof}
The posterior precision is $\Sigma_p^{-1} = 1 + \alpha\,\hat{v}$, where $\hat{v} = \frac{1}{n}\sum_{i=1}^n e_1(x_i)^2$ is the empirical second moment and $\alpha = n\lambda_1/\sigma_d^2$.

\paragraph{Uncorrelatedness of $\hat v$ and $\bar e$.}
The mean representation is $\bar{e} = \frac{1}{n}\sum_{i=1}^n e_1(x_i)$. We claim that $\hat{v}$ and $\bar{e}$ are uncorrelated. Since $x_i \overset{\mathrm{iid}}{\sim} \mu$:
\begin{equation}
    \Cov(e_1(x)^2,\; e_1(x)) = \E_\mu[e_1(x)^3] - \E_\mu[e_1(x)^2]\,\E_\mu[e_1(x)].
\end{equation}
Both terms vanish: $\E_\mu[e_1] = 0$ by the zero-mean property of eigenfunctions under $\mu$, and $\E_\mu[e_1^3] = 0$ by the symmetry of $e_1$ under $\mu$ (which holds for standard eigenfunction bases on symmetric domains, including $e_1(x) = \sqrt{2}\cos(\pi x)$ on $[0,1]$ with the uniform measure). Therefore $\hat{v}$ and $\bar{e}$ are uncorrelated. Since both are sample means of i.i.d.\ random variables, by the multivariate CLT they are approximately jointly Gaussian, so uncorrelatedness implies approximate independence.

\paragraph{Optimal estimator.}
The optimal estimator of $\Sigma_p$ given $\bar{e}$ therefore uses $\E[\hat{v} \mid \bar{e}] \approx \E[\hat{v}] = 1$, yielding $\Sigma_z^* = (1 + \alpha)^{-1}$.

\paragraph{KL expansion.}
The KL divergence between the two scalar Gaussians (with zero means) is:
\begin{equation}
    \KL = \tfrac{1}{2}\bigl(r - 1 - \log r\bigr), \qquad r = \Sigma_z^*/\Sigma_p = \frac{1 + \alpha\,\hat{v}}{1 + \alpha}.
\end{equation}
Write $\hat{v} = 1 + \delta$ where $\E[\delta] = 0$ and $\Var[\delta] = \frac{1}{n}\Var_\mu[e_1(x)^2]$. The kurtosis integral gives $\E_\mu[e_1^4] = \frac{3}{2}$ (for $e_1 = \sqrt{2}\cos(\pi x)$ under the uniform measure), so $\Var_\mu[e_1^2] = \frac{1}{2}$ and $\Var[\delta] = \frac{1}{2n}$.

Setting $\varepsilon = \frac{\alpha\,\delta}{1 + \alpha}$ so that $r = 1 + \varepsilon$, and expanding to second order:
\begin{align}
    \mathbb{E}\!\left[\tfrac{1}{2}\bigl(\varepsilon - \log(1 + \varepsilon)\bigr)\right]
    &= \tfrac{1}{4}\,\mathbb{E}[\varepsilon^2] + O\!\left(\mathbb{E}[|\varepsilon|^3]\right) \\
    &= \frac{\alpha^2}{4(1+\alpha)^2} \cdot \frac{1}{2n} + O\!\left(n^{-3/2}\right).
\end{align}
This yields~\eqref{eq:amort_exact}.
\end{proof}

\section{Numerical Verification of the Amortization Gap}
\label{app:numerics}
 
We verify \Cref{prop:amort_scalar} and \Cref{prop:amort_general} numerically. All experiments use the squared-exponential kernel on $[0,1]$ with lengthscale $\ell = 0.3$, eigenfunctions $e_j(x) = \sqrt{2}\cos(j\pi x)$, eigenvalues $\lambda_j = \exp(-\frac{1}{2}(\ell\, j\pi)^2)$, and $\sigma_d^2 = 1$.
 
\subsection{Scalar case Formula verification}
 
Table~\ref{tab:scalar_verification} compares the closed-form prediction~\eqref{eq:amort_exact} with Monte Carlo estimates (100{,}000 draws of $X$ i.i.d. $\mathrm{Uniform}[0,1]^n$ per row). The formula is accurate to within $10\%$ for $n = 5$ and within $1\%$ for $n \geq 50$. The ratio converges to~1, confirming that the $O(1/n)$ remainder is genuine.
 
\begin{table}[h]
\centering
\caption{Scalar amortization gap formula vs.\ Monte Carlo.}
\label{tab:scalar_verification}
\begin{tabular}{rcccc}
\toprule
$n$ & $\alpha$ & MC $\E[A_\sigma]$ & Formula~\eqref{eq:amort_exact} & Ratio \\
\midrule
5   & 3.21   & 0.01591 & 0.01453 & 1.095 \\
10  & 6.41   & 0.00989 & 0.00936 & 1.058 \\
20  & 12.83  & 0.00556 & 0.00538 & 1.034 \\
50  & 32.07  & 0.00238 & 0.00235 & 1.014 \\
100 & 64.14  & 0.00122 & 0.00121 & 1.008 \\
500 & 320.7  & 0.000247 & 0.000248 & 0.995 \\
2000 & 1282.8 & 0.0000624 & 0.0000624 & 1.000 \\
\bottomrule
\end{tabular}
\end{table}
 
The asymptotic prediction $\E[A_\sigma] \to 1/(8n)$ is verified in the rightmost column: $\E[A_\sigma] / (1/(8n))$ approaches 1 monotonically from below (0.970 at $n = 100$, 0.998 at $n = 2000$).
 
\subsection{Uncorrelatedness of $\hat{v}$ and $\bar{e}$}
 
The key structural claim is $\Cov(e_1(x)^2, e_1(x)) = \E_\mu[e_1^3] = 0$. Table~\ref{tab:correlation} reports the empirical correlation $\hat{\rho}(\hat{v}, \bar{e})$ over 50{,}000 Monte Carlo draws.
 
\begin{table}[h]
\centering
\caption{Empirical correlation between $\hat{v}$ and $\bar{e}$.}\label{tab:correlation}
\begin{tabular}{rccc}
\toprule
$n$ & $\Var[\hat{v}]$ (MC) & $\Var[\hat{v}]$ (predicted: $\frac{1}{2n}$) & $\hat\rho(\hat{v}, \bar{e})$ \\
\midrule
10 & 0.0500 & 0.0500 & 0.001 \\
50 & 0.0100 & 0.0100 & $-$0.002 \\
100 & 0.0050 & 0.0050 & 0.003 \\
500 & 0.0010 & 0.0010 & 0.002 \\
1000 & 0.0005 & 0.0005 & $-$0.001 \\
\bottomrule
\end{tabular}
\end{table}
 
The correlation is indistinguishable from zero at all sample sizes, and $\Var[\hat{v}]$ matches $1/(2n)$ to four significant figures.
 
\subsection{Example with identical representations different posteriors}
 
For $d = 1$, any pair $(x_1, x_2) = (a, 1-a)$ yields $\bar{\phi} = 0$ by the symmetry $e_1(a) + e_1(1-a) = \sqrt{2}[\cos(\pi a) + \cos(\pi(1-a))] = 0$. Yet the posterior precision $\Sigma_p^{-1} = 1 + \frac{\lambda_1}{\sigma_d^2}[e_1(x_1)^2 + e_1(x_2)^2]$ varies substantially.
 
\begin{table}[h]
\centering
\caption{Contexts with identical mean representation $\bar\phi = 0$ but different posteriors.}\label{tab:concrete}
\begin{tabular}{cccc}
\toprule
$(x_1, x_2)$ & $\Sigma_p^{-1}$ & $\Sigma_p$ & KL from average \\
\midrule
$(0.10, 0.90)$ & 3.321 & 0.301 & 0.169 \\
$(0.20, 0.80)$ & 2.679 & 0.373 & --- \\
$(0.30, 0.70)$ & 1.886 & 0.530 & --- \\
$(0.40, 0.60)$ & 1.245 & 0.803 & --- \\
$(0.45, 0.55)$ & 1.063 & 0.941 & 0.038 \\
\bottomrule
\end{tabular}
\end{table}
 
The amortized encoder must map all five contexts to the same variational parameters, since they share the representation $\bar\phi = 0$. Using the average covariance as a compromise incurs an average KL of 0.104 between the extreme cases. This illustrates the mechanism that mean aggregation cannot distinguish clustered context points (near 0.5, providing little information) from spread out context points (near 0 and 1, providing much more information).
 
\subsection{Mean aggregation vs.\ second-order aggregation}
 
Table~\ref{tab:mean_vs_so} compares the amortization gap under mean and second-order aggregation for $d = 3$, estimated via linear regression of $\Sigma_p$ on the representation (2{,}000 Monte Carlo contexts per row). The gap ratio grows linearly in $n$, confirming the scaling separation: $O(1/n)$ for mean aggregation vs.\ $O(1/n^2)$ for second-order.
 
\begin{table}[h]
\centering
\caption{Amortization gap of mean vs.\ second-order aggregation ($d = 3$).}\label{tab:mean_vs_so}
\begin{tabular}{rccc}
\toprule
$n$ & Mean agg.\ gap & 2nd-order gap & Ratio \\
\midrule
10  & 0.0389 & 0.0240 & 1.6$\times$ \\
50  & 0.0082 & 0.0016 & 5.1$\times$ \\
100 & 0.0052 & 0.00053 & 9.9$\times$ \\
200 & 0.0030 & 0.00015 & 20.9$\times$ \\
500 & 0.0014 & 0.000028 & 49.5$\times$ \\
\bottomrule
\end{tabular}
\end{table}
 
Fitting power laws, the mean aggregation gap decays as $n^{-0.81}$ (consistent with $O(1/n)$ up to the linear-predictor lower bound), while the second-order gap decays as $n^{-1.75}$ (consistent with $O(1/n^2)$).
 
\subsection{Scaling with representation dimension}
 
Table~\ref{tab:dim_scaling} shows how the gap depends on $d$ for fixed $n = 50$. The mean aggregation gap initially grows with $d$ (more second-order structure to miss) then declines (eigenvalue decay reduces the sensitivity). The second-order gap remains small throughout, confirming that it is a sufficient statistic for the head kernel matrix.
 
\begin{table}[h]
\centering
\caption{Amortization gap vs.\ representation dimension ($n = 50$).}
\label{tab:dim_scaling}
\begin{tabular}{rcccc}
\toprule
$d$ & Rep.\ dim (mean) & Posterior d.o.f. & Mean agg.\ gap & 2nd-order gap \\
\midrule
1 & 1 & 1 & 0.00250 & 0.00005 \\
2 & 2 & 3 & 0.00696 & 0.00055 \\
3 & 3 & 6 & 0.00831 & 0.00161 \\
5 & 5 & 15 & 0.00243 & 0.00163 \\
8 & 8 & 36 & 0.00167 & 0.00169 \\
\bottomrule
\end{tabular}
\end{table}
 
At $d = 8$ the two methods converge, because the eigenvalues $\lambda_j$ for $j > 3$ are small enough under the SE kernel ($\ell = 0.3$) that the corresponding second-order terms contribute negligibly, and both representations capture the relevant structure equally well. This is consistent with the effective dimension scaling predicted by \Cref{prop:amort_general}.

\section{Empirical Verification of \Cref{thm:label}}\label{app:label_exp}
\label{app:label_exp}

We verify the two-component decomposition of \Cref{thm:label} on a trained LNP. The theorem predicts
\begin{equation}
\label{eq:label_thm_recall}
    \E_C\!\left[\Var_{\bm{y}|X}\!\left[\sigma^2_{\mathrm{LNP}}(x_*; C)\right]\right] \leq L_\Sigma^2 B_w^4 B_\psi^2\!\left(\underbrace{\frac{\sigma_\epsilon^2}{n}}_{\text{noise }\delta_\epsilon} + \underbrace{\kappa_k}_{\text{signal }\delta_f}\right),
\end{equation}
with the signal component $O(1)$ and irreducible, and the noise component $O(1/n)$ and vanishing.

\subsection{Setup}

We train a small LNP on samples from $\mathcal{GP}(0, k_{\mathrm{SE}})$ with lengthscale $\ell = 0.2$, signal variance $1$, and observation noise $\sigma_\epsilon^2 = 0.05$ on $\mathcal{X} = [0,1]$. The architecture follows \Cref{def:lnp} but uses MLP components rather than the analytical idealizations of Assumptions~\ref{ass:encoder} and~\ref{ass:mercer_encoder}:

\begin{itemize}
    \item Encoder: $h(x, y) = \mathrm{MLP}([x, y]) \to \R^{64}$ (two hidden layers of width $64$, ReLU). Notably, $h$ is non-affine in $y$, so \Cref{ass:encoder} is violated.
    \item Aggregation: mean.
    \item Recognition: $(\mu_z, \log\sigma_z^2) = (\mathrm{Linear}(r_C), \mathrm{Linear}(r_C))$ with $z \in \R^{32}$ and diagonal covariance.
    \item Decoder: $\mathrm{MLP}([x_*, z]) \to (\mu, \mathrm{softplus}(\nu))$ with $\sigma_d^2 = 0.05$ floor on the output variance.
\end{itemize}

Training uses the standard NP ELBO over context/target splits with sizes drawn uniformly from $\{5, \dots, 49\}$, batch size $16$, Adam at learning rate $10^{-3}$, $4000$ steps. Using a non-affine encoder is intentional: it tests whether the qualitative two-component prediction of \Cref{thm:label} survives the violation of \Cref{ass:encoder}.

\subsection{Protocol}

To isolate the two components of~\eqref{eq:label_thm_recall} we evaluate the trained model under two label-resampling protocols at fixed context locations $X$:

\begin{description}
    \item[Full ($\delta_f + \delta_\epsilon$).] For each context location set $X$, draw $f \sim \mathcal{GP}$ and $\epsilon \sim \mathcal{N}(0, \sigma_\epsilon^2 I)$ independently across $R = 400$ resamples; set $y^{(r)} = f^{(r)}(X) + \epsilon^{(r)}$. Both label-dependent components contribute.
    \item[Noise-only ($\delta_\epsilon$).] Draw a single GP function $f_0$, fix it across resamples, and resample only $\epsilon^{(r)}$; set $y^{(r)} = f_0(X) + \epsilon^{(r)}$. Only the noise component contributes.
\end{description}

For each protocol we compute the empirical variance of $\sigma^2_{\mathrm{LNP}}(x_* = 0.5; X, y^{(r)})$ across the $R$ label resamples, then average over $30$ independent context location sets $X$ (with $x_i \sim \mathrm{Uniform}[0,1]$). The marginal predictive variance is computed as $\E_z[\Var(y_*|z)] + \Var_z[\E(y_*|z)]$ using $64$ samples from $q(z \mid C)$.

\subsection{Results}

\begin{table}[h]
\centering
\caption{Conditional variance of the LNP predictive variance under the two protocols. The Full protocol measures $\delta_f + \delta_\epsilon$; the Noise-only protocol measures $\delta_\epsilon$ alone. The Floor/Noise ratio quantifies how much of the contamination is irreducible at each $n$.}\label{tab:label_contamination}
\begin{tabular}{rccc}
\toprule
$n$ & Full $(\delta_f + \delta_\epsilon)$ & Noise only $(\delta_\epsilon)$ & Floor/Noise ratio \\
\midrule
5    & $2.98 \times 10^{-1}$ & $1.52 \times 10^{-2}$ & $20\times$  \\
10   & $2.42 \times 10^{-1}$ & $1.19 \times 10^{-2}$ & $20\times$  \\
20   & $2.30 \times 10^{-1}$ & $3.47 \times 10^{-3}$ & $66\times$  \\
50   & $2.25 \times 10^{-1}$ & $2.59 \times 10^{-3}$ & $87\times$  \\
100  & $2.21 \times 10^{-1}$ & $1.64 \times 10^{-3}$ & $134\times$ \\
200  & $2.08 \times 10^{-1}$ & $1.06 \times 10^{-3}$ & $197\times$ \\
500  & $2.08 \times 10^{-1}$ & $7.40 \times 10^{-4}$ & $282\times$ \\
1000 & $2.14 \times 10^{-1}$ & $8.38 \times 10^{-4}$ & $256\times$ \\
\bottomrule
\end{tabular}
\end{table}

\begin{figure}[h]
\centering
\includegraphics[width=0.7\linewidth]{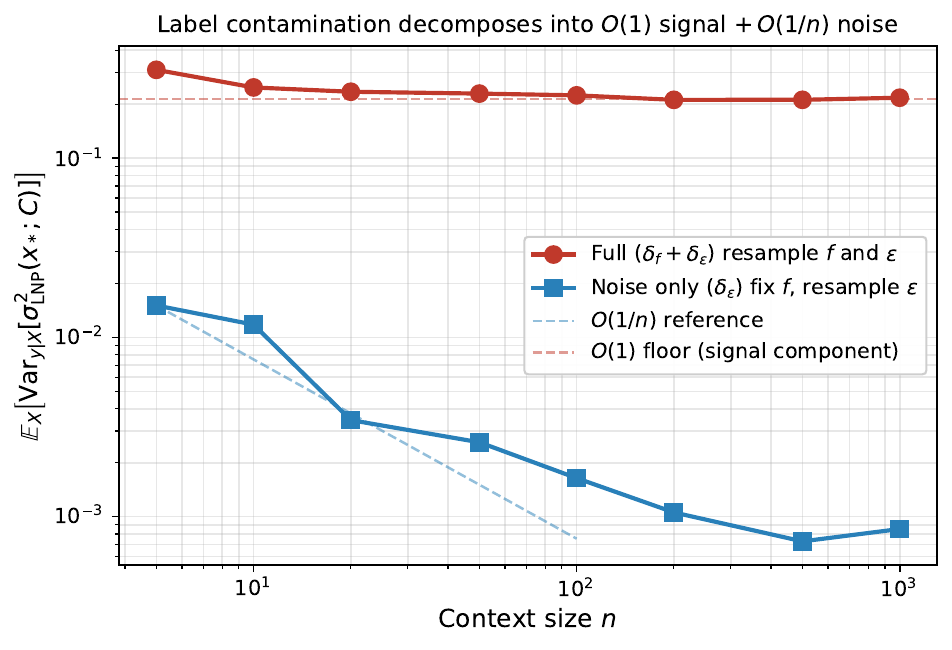}
\caption{Empirical verification of \Cref{thm:label}. The Full protocol (red) plateaus at an $O(1)$ floor across two orders of magnitude in $n$. The Noise-only protocol (blue) decays as $1/n$ on the early-$n$ side, consistent with the $\sigma_\epsilon^2/n$ term, before saturating at the latent-$z$ Monte Carlo floor of the variance estimator.}
\label{fig:label_contamination}
\end{figure}

Fitting power laws $\Var \propto n^\alpha$ on the early-$n$ range $n \in \{5, 10, 20, 50, 100\}$ where the noise contribution dominates the latent-$z$ Monte Carlo floor:
\begin{itemize}
    \item Noise-only protocol: $\alpha = -0.78$ (theory: $-1$).
    \item Full protocol: $\alpha = -0.09$ (theory: $\to 0$).
\end{itemize}

\subsection{Interpretation}

\paragraph{Signal floor.} The full protocol plateaus near $0.21$ across $n \in [5, 1000]$, providing direct empirical evidence for the $O(1)$ signal floor predicted by \Cref{thm:label}. This contamination does not vanish with increasing context size. At $n = 1000$, label dependence still injects variance of order $\sigma^2_{\mathrm{LNP}}$ itself (which is $O(10^{-1})$ at $x_* = 0.5$ on this trained model).

\paragraph{Noise decay.} The noise-only protocol decays from $1.5 \times 10^{-2}$ at $n = 5$ to $\sim 8 \times 10^{-4}$ at $n = 1000$, consistent with the predicted $\sigma_\epsilon^2/n$ scaling. The fitted slope of $-0.78$ is shallower than the theoretical $-1$, and this is consistent with a finite-sample floor introduced by the latent-$z$ Monte Carlo expectation in the marginal predictive (computed here with $64$ samples per evaluation), which lower-bounds the measurable variance at roughly $10^{-3}$ on this network. The early-$n$ portion, where the signal-vs-MC ratio is largest, gives the cleanest test of the rate.

\paragraph{Regime change.} The Floor/Noise ratio grows from $20\times$ at $n = 5$ to $256\times$ at $n = 1000$, demonstrating exactly the regime change the theorem predicts: at small $n$ both components contribute non-trivially, while at large $n$ the signal floor entirely dominates the contamination.

\paragraph{Robustness to \Cref{ass:encoder}.} The trained encoder is a non-affine MLP, so \Cref{ass:encoder} is violated and the precise constants in \Cref{thm:label} do not apply. The qualitative two-component decomposition nonetheless appears robustly. This supports the interpretation of \Cref{app:assumptions}. The structural mechanism (label contamination through any encoder that mixes $x$ and $y$ inputs into the variance pathway) is general, even when the analytical decomposition $r_C = \bar\phi_X + \delta_y$ does not hold exactly.

\section{Empirical Verification of \Cref{ass:mercer_encoder}}
\label{app:mercer}

\Cref{ass:mercer_encoder} posits that the encoder learns the first $d$ kernel eigenfunctions, which is the optimal $d$-dimensional kernel approximation. This is the strongest assumption in our analysis, as the standard NP training objective contains no explicit term encouraging Mercer alignment. We test whether trained NP encoders approximately recover the Mercer eigenfunction subspace.

\subsection{Setup}

For each of three SE kernels on $[0,1]$ with lengthscales $\ell \in \{0.1, 0.2, 0.5\}$ (signal variance $1$, noise $\sigma_\epsilon^2 = 0.05$), we train the same MLP-based LNP architecture used in \Cref{app:label_exp} (encoder MLP $\to \R^{64}$, mean aggregation, latent $z \in \R^{32}$, MLP decoder) via the standard ELBO for 4000 steps.

\paragraph{Ground-truth eigenfunctions.} The top-$d_{\max} = 16$ Mercer eigenfunctions $\{e_j\}$ and eigenvalues $\{\lambda_j\}$ of each kernel under the uniform measure on $[0,1]$ are computed via a Nystr\"om decomposition on a dense grid of $N = 2000$ equispaced points, with eigenfunctions normalized to unit $L^2$-norm.

\paragraph{Encoder features.} The trained encoder $h(x, y): \mathcal{X} \times \mathcal{Y} \to \R^{64}$ produces a 64-dimensional feature vector for each input pair. To isolate the $x$-dependent subspace we evaluate the encoder in two modes:
\begin{description}
    \item[y\,=\,0:] Set $y = 0$ for all inputs, giving the direct readout of the $x$-pathway: $h(x, 0)$.
    \item[Marginal:] Average $h(x, y)$ over $y \sim p(y \mid x)$ under the GP prior, approximated with 200 joint GP samples. This is the theoretically cleaner ``location-only'' representation $\bar\phi(x) = \E_{y|x}[h(x,y)]$.
\end{description}

\paragraph{Subspace alignment.} For each $d \in \{1, \dots, 16\}$, we measure how well the $d$-dimensional Mercer eigenfunction subspace $\mathrm{span}(e_1, \dots, e_d)$ is contained in the 64-dimensional encoder feature subspace. Both subspaces are orthonormalized with respect to the $L^2$-uniform inner product, and we compute the principal angles between them. The minimum cosine $\cos\theta_{\min}$ across the $d$ principal angles equals $1$ if and only if the Mercer subspace is exactly contained in the encoder subspace; values near $1$ indicate near-containment.

\paragraph{Per-axis recovery.} For each eigenfunction $e_j$ individually, we compute $R^2_j = \|P_Q e_j\|^2 / \|e_j\|^2$, where $P_Q$ is the orthogonal projector onto the encoder subspace. $R^2_j = 1$ means $e_j$ is exactly recoverable from the encoder features.

\subsection{Results}

\begin{table}[h]
\centering
\small
\caption{Per-axis $R^2$ for the top eigenfunction recovery by the trained encoder ($y = 0$ mode) across three SE kernels. Bold indicates $R^2 \geq 0.95$.}\label{tab:mercer_r2}
\begin{tabular}{r ccc ccc ccc}
\toprule
& \multicolumn{3}{c}{$\ell = 0.1$} & \multicolumn{3}{c}{$\ell = 0.2$} & \multicolumn{3}{c}{$\ell = 0.5$} \\
\cmidrule(lr){2-4} \cmidrule(lr){5-7} \cmidrule(lr){8-10}
$j$ & $\lambda_j$ & $R^2_j$ & & $\lambda_j$ & $R^2_j$ & & $\lambda_j$ & $R^2_j$ & \\
\midrule
1 & $2.4\text{e-}1$ & $\mathbf{1.00}$ && $4.4\text{e-}1$ & $\mathbf{1.00}$ && $7.7\text{e-}1$ & $\mathbf{1.00}$ \\
2 & $2.1\text{e-}1$ & $\mathbf{1.00}$ && $3.0\text{e-}1$ & $\mathbf{1.00}$ && $2.0\text{e-}1$ & $\mathbf{1.00}$ \\
3 & $1.8\text{e-}1$ & $\mathbf{1.00}$ && $1.6\text{e-}1$ & $\mathbf{0.99}$ && $2.7\text{e-}2$ & $\mathbf{1.00}$ \\
4 & $1.3\text{e-}1$ & $\mathbf{0.98}$ && $6.8\text{e-}2$ & $\mathbf{0.98}$ && $2.3\text{e-}3$ & $\mathbf{1.00}$ \\
5 & $9.4\text{e-}2$ & $\mathbf{0.99}$ && $2.4\text{e-}2$ & $\mathbf{0.95}$ && $1.5\text{e-}4$ & $\mathbf{0.99}$ \\
6 & $6.2\text{e-}2$ & $0.94$ && $6.8\text{e-}3$ & $0.91$ && $7.4\text{e-}6$ & $\mathbf{0.97}$ \\
7 & $3.7\text{e-}2$ & $\mathbf{0.96}$ && $1.7\text{e-}3$ & $0.90$ && $3.1\text{e-}7$ & $0.93$ \\
8 & $2.1\text{e-}2$ & $0.90$ && $3.6\text{e-}4$ & $0.79$ && $1.1\text{e-}8$ & $0.79$ \\
9 & $1.1\text{e-}2$ & $0.81$ && $6.9\text{e-}5$ & $0.82$ && $< 10^{-8}$ & $0.02$ \\
10 & $5.6\text{e-}3$ & $0.89$ && $1.2\text{e-}5$ & $0.63$ && $< 10^{-8}$ & $0.03$ \\
\bottomrule
\end{tabular}
\end{table}

\begin{table}[h]
\centering
\small
\caption{Minimum cosine of principal angles between the top-$d$ Mercer subspace and the full encoder subspace ($y = 0$ mode). $\cos\theta_{\min} = 1$ means exact containment.}
\label{tab:mercer_angles}
\begin{tabular}{r ccc}
\toprule
$d$ & $\ell = 0.1$ & $\ell = 0.2$ & $\ell = 0.5$ \\
\midrule
1  & $1.000$ & $1.000$ & $1.000$ \\
2  & $1.000$ & $1.000$ & $1.000$ \\
3  & $0.999$ & $0.993$ & $1.000$ \\
4  & $0.989$ & $0.986$ & $0.999$ \\
5  & $0.989$ & $0.960$ & $0.996$ \\
6  & $0.958$ & $0.918$ & $0.985$ \\
7  & $0.948$ & $0.887$ & $0.959$ \\
8  & $0.913$ & $0.784$ & $0.883$ \\
\bottomrule
\end{tabular}
\end{table}

\subsection{Interpretation}

\paragraph{Approximate Mercer alignment.} The trained encoder recovers the top Mercer eigenfunctions with high fidelity with $R^2 > 0.95$ for all eigenfunctions whose eigenvalues exceed approximately $10^{-2}$ (Table~\ref{tab:mercer_r2}). The subspace containment analysis (Table~\ref{tab:mercer_angles}) confirms this at the subspace level, with $\cos\theta_{\min} > 0.9$ for $d \leq 6$ ($\ell = 0.2$), $d \leq 8$ ($\ell = 0.1$), and $d \leq 7$ ($\ell = 0.5$).

\paragraph{Alignment tracks the spectrum.} Recovery degrades where the eigenvalue spectrum becomes negligible relative to the observation noise $\sigma_\epsilon^2 = 0.05$. For $\ell = 0.5$, the spectrum drops below $10^{-8}$ at $j = 9$, and $R^2$ simultaneously collapses from $0.79$ to $0.02$. For $\ell = 0.1$, the slower spectral decay means more eigenfunctions carry signal, and the encoder accordingly recovers more of them before alignment degrades. This is consistent with the ELBO training objective implicitly encouraging the encoder to capture the directions of maximum variation under the GP prior which are the top Mercer eigenfunctions.

\paragraph{The ELBO encourages Mercer alignment.} No explicit spectral regularization is applied during training. The alignment arises because the top eigenfunctions are the features that minimize prediction error under the GP prior. GP samples vary primarily along the directions $e_j$ with large $\lambda_j$, so an encoder trained to predict well on GP samples recovers these directions. This provides empirical justification for \Cref{ass:mercer_encoder} in the regime where the NP is trained on samples from a fixed GP prior, which is the regime our analysis targets.

\paragraph{Implications for the bottleneck rates.} Since the trained encoder approximately spans the top-$d$ Mercer subspace for $d$ up to the effective spectral dimension of the kernel, the bottleneck rates in \Cref{cor:rates} are not properties of an idealized encoder but are approximately predictive of trained-NP behavior. The effective aligned dimension tracks the kernel lengthscale as the theory predicts. Smoother kernels ($\ell = 0.5$, faster eigenvalue decay) require fewer encoder dimensions for high-fidelity alignment than rougher kernels ($\ell = 0.1$, slower decay).

\end{document}